\documentclass{article}

\usepackage{microtype}
\usepackage{multirow}
\usepackage{graphicx}
\usepackage{subfigure}
\usepackage{booktabs} 
\usepackage{caption}
\usepackage{wrapfig}
\usepackage{xcolor}  
\usepackage{hyperref}

\usepackage[accepted]{icml2026} 

\usepackage{amsmath}
\usepackage{amssymb}
\usepackage{mathtools}
\usepackage{algorithm}
\usepackage{algorithmic}
\usepackage{amsmath} 
\usepackage{enumitem}
\usepackage{bm}

\DeclareMathOperator*{\argmax}{arg\,max}

\DeclareMathOperator{\MAP}{MAP}

\DeclareMathOperator{\Cov}{Cov}

\DeclareMathOperator{\vecc}{vec}

\def\Dcal{{\mathcal{D}}}

\def\Lcal{{\mathcal{L}}}

\def\Ncal{{\mathcal{N}}}
\def\Ocal{{\mathcal{O}}}

\def\Ycal{{\mathcal{Y}}}

\def\Ebb{{\mathbb{E}}}

\def\Rbb{{\mathbb{R}}}

\def\pbm{{\bm{p}}}

\def\xbm{{\bm{x}}}

\def\zbm{{\bm{z}}}
\def\alphabm{{\bm{\alpha}}}
\def\pibm{{\bm{\pi}}}
\def\mubm{{\bm{\mu}}}
\def\Sigmabm{{\bm{\Sigma}}}

\def\b0{{\mathbf{0}}}

\usepackage[capitalize,noabbrev]{cleveref}

\usepackage{amsthm}

\theoremstyle{plain}
\newtheorem{theorem}{Theorem}[section]
\newtheorem{proposition}[theorem]{Proposition}

\theoremstyle{definition}

\theoremstyle{remark}

\usepackage{dsfont}
\newcommand{\softmax}{\operatorname{softmax}}
\newcommand{\KL}{\mathrm{KL}}

\usepackage[textsize=tiny]{todonotes}

\begin{document}

\twocolumn[
\icmltitle{Laplace-Bridged Randomized Smoothing for Fast Certified Robustness}








\begin{icmlauthorlist}
\icmlauthor{Miao Lin}{oducs}
\icmlauthor{MD Saifur Rahman Mazumder}{utepmath}
\icmlauthor{Feng Yu}{utepmath}
\icmlauthor{Daniel Takabi}{oducyber}
\icmlauthor{Rui Ning}{oducs}
\end{icmlauthorlist}

\icmlaffiliation{oducs}{Department of Computer Science, Old Dominion University, Norfolk, VA, USA}
\icmlaffiliation{utepmath}{Department of Mathematical Sciences, University of Texas at El Paso, El Paso, TX, USA}
\icmlaffiliation{oducyber}{School of Cybersecurity, Old Dominion University, Norfolk, VA, USA}

\vspace{0.08in}
\begin{center}
{\small
$^{1}$ Department of Computer Science, Old Dominion University, Norfolk, VA, USA\\
$^{2}$ Department of Mathematical Sciences, University of Texas at El Paso, El Paso, TX, USA\\
$^{3}$ School of Cybersecurity, Old Dominion University, Norfolk, VA, USA
}
\end{center}

\icmlkeywords{Machine Learning, ICML}

\vskip 0.3in
]

\begin{abstract}
\emph{Randomized Smoothing (RS)} offers formal $\ell_2$ guarantees for arbitrary base classifiers but faces two key practical bottlenecks: (i) it often relies on \emph{noise-augmented training} to achieve nontrivial certificates, which increases training cost, can reduce clean accuracy, and weakens RS as a genuinely post-hoc defense; and (ii) certification is computationally expensive, typically requiring tens of thousands of noisy forward passes per input, which hinders deployment, especially on resource-constrained edge devices.
To address both limitations, we propose \emph{Laplace-Bridged Smoothing (LBS)}, an analytic reformulation of RS that replaces high-dimensional input-space Monte Carlo (MC) sampling with efficient computations in a low-dimensional probability space. 
LBS preserves formal robustness guarantees without requiring noise-augmented training while substantially reducing certification burden. 
On CIFAR-10 and ImageNet, LBS attains stronger certified robustness than RS and reduces per-sample certification cost by nearly an order of magnitude. 
Notably, on NVIDIA Jetson Orin Nano and Raspberry Pi 4, LBS achieves speedups of up to $494\times$, enabling practical certified deployment on real-world edge devices. 
Finally, we provide theoretical justification for the analytic formulation and certificate validity of LBS. 
\end{abstract}

\section{Introduction}

Deep neural networks (DNNs) have achieved remarkable success across many applications, yet they remain vulnerable to \emph{adversarial perturbations}, small often imperceptible input changes that can drastically alter predictions \citep{szegedy2013intriguing, goodfellow2014explaining,kurakin2018adversarial, akhtar2021advances,han2023interpreting, liu2025comprehensive}. A large body of work improves DNN robustness via adversarial training \citep{goodfellow2014explaining,madry2017towards,shafahi2019adversarial, zhang2019theoretically, wang2024revisiting}, which is effective in practice but lacks formal guarantees, incurs high computational cost from iterative attacks, and can be circumvented by adaptive attacks \citep{athalye2018obfuscated,uesato2018adversarial}. To obtain sound guarantees, the community has increasingly turned to \emph{certified defenses}, which provide formal robustness guarantees against adversarial perturbations \citep{hein2017formal, raghunathan2018semidefinite, cohen2019certified,jeong2023confidence}.

Among certified defenses, \emph{Randomized Smoothing (RS)} is widely used for its simplicity and compatibility with arbitrary base classifiers. By injecting random noise into the input, RS guarantees prediction invariance to all adversarial perturbations within an $\ell_p$ radius~\citep{lecuyer2019certified,cohen2019certified,jia2019certified,salman2019provably,blum2020random,yang2020randomized,wang2020certifying,li2023sok,scholten2023hierarchical,wang2024drf}.
Despite this generality, RS faces two practical limitations in real-world applications.
First, meaningful certificates typically require the base classifier to be robust under noise, making RS heavily dependent on noise-augmented training, which increases training cost, can degrade clean accuracy, and weakens RS's appeal as a truly post-hoc defense (\Cref{sec:RS}). 
Second, RS relies on high-dimensional input-space MC sampling, often with thousands to tens of thousands of noisy forward passes per input, making certification expensive and difficult to scale, \emph{especially on resource-constrained devices where limited compute and memory preclude effective batching or parallelism} (\Cref{sec:RS}). Yet in many security- and safety-critical deployments (e.g., edge computing, robotics, and autonomous vehicles)~\citep{liu2019edge, kumar2024edge, huckelberry2024tinyml}, decisions must be certified \emph{at the point of use} to support real-time operation and meet system-level assurance requirements~\citep{tambon2022certify, sridhar2025approach}, which further amplifies the need for efficient on-device RS certification.
To address both limitations of RS, we introduce \emph{Laplace-Bridged Smoothing (LBS)}, a unified framework that combines RS with the Laplace-Bridge (LB) approximation to enable efficient and statistically grounded robustness certification within an $\ell_2$ radius. LBS analytically propagates input noise through a locally linearized feature extractor to approximate the noisy logit distribution, then applies a Gaussian–Dirichlet LB to derive a tractable Dirichlet surrogate for the smoothed predictive distribution. To estimate the top-class probability, LBS performs MC sampling directly in this low-dimensional Dirichlet space and applies a one-sided Beta quantile to obtain a certified lower bound. By shifting stochastic estimation from the high-dimensional input space to the low-dimensional probability simplex, LBS preserves statistical rigor while substantially reducing computational cost relative to standard RS, \textit{its advantage being especially pronounced in on-device settings, where eliminating repeated forward passes reduces the need for sequential inference and alleviates compute and memory bottlenecks} (Section~\ref{edge} and Appendix~\ref{edge_app}).
\textbf{Our contributions are fourfold:}

(1) We propose \emph{Laplace-Bridged Smoothing (LBS)}, an analytic reformulation of RS that replaces high-dimensional input-space sampling with low-dimensional Dirichlet sampling via a Gaussian–Dirichlet bridge, enabling statistically sound $\ell_2$ certification;

(2) We establish the theoretical foundation of LBS by deriving a closed-form mapping from linearized Gaussian logit statistics to Dirichlet parameters $\alpha$, yielding a tractable surrogate for the smoothed predictive distribution;

(3) On CIFAR-10 and ImageNet across diverse architectures, LBS improves certified accuracy while reducing certification cost by up to an order of magnitude; 

(4) On resource-constrained devices, including NVIDIA Jetson Orin Nano and Raspberry~Pi~4, LBS achieves up to $494\times$ speedups over RS, enabling practical on-device certification under tight compute and memory budgets.

\section{Related Work}\label{sec:relatedwork}


\paragraph{Certified defenses.}  
Certified defenses aim to provide formal robustness guarantees against adversarial perturbations. Early work focused on exact or approximate verification through convex relaxations or semidefinite programming \citep{ lomuscio2017approach, raghunathan2018semidefinite, wong2018provable} or bounding Lipschitz constant~\citep{hein2017formal, tsuzuku2018lipschitz, anil2019sorting}. Despite strong theoretical grounding, these approaches often suffer from scalability issues when applied to modern large-scale architectures.  

\vspace{-3mm}\paragraph{Randomized smoothing (RS).}  
RS has emerged as a widely adopted certified defense. \citet{lecuyer2019certified} established robustness guarantees for smoothed classifiers via a differential privacy perspective, \citet{cohen2019certified} provided tight $\ell_2$ guarantees by constructing Gaussian-smoothed classifiers. Subsequent works extended RS using adversarially trained base models~\citep{salman2019provably,nandy2021towards,wang2024drf} and attack-free objectives that directly maximize certified radii, such as MACER~\citep{zhai2020macer}. Extensions beyond $\ell_2$ to general $\ell_p$ norms have also been explored~\citep{blum2020random,tengl1,li2022double,li2023sok}. \emph{In practice, RS typically requires noise-augmented training to obtain nontrivial certificates.} Moreover, statistically valid certification relies on large-scale MC sampling, incurring substantial computational overhead and limiting scalability to large datasets and architectures. 

\vspace{-3mm}\paragraph{Accelerating certification.}
Several studies seek to reduce the computational cost of RS. \citet{feng2020regularized} proposes a regularized training scheme to lower certification cost, but it can introduce bias and degrade accuracy. \citet{chen2022input} dynamically allocates MC samples during certification, while \citet{seferis2023randomized,seferis2024estimating} observes that relatively few samples can already yield nontrivial certificates. \citet{ugare2023incremental} accelerates certification after model post-processing. More recently, \citet{bhardwaj2024accelerated} replace MC sampling with a learned surrogate for class count prediction, achieving empirical speedups but lacking formal guarantees, and \citet{voracek2024treatment} improve MC efficiency via optimal confidence intervals and confidence sequences. Despite these advances, existing methods still depend on extensive sampling or additional training assumptions; moreover, \textbf{\textit{all of these approaches rely on noise-augmented training of the base classifier}}, motivating the need for more efficient, post-hoc, and theoretically grounded certification, which is the focus of this work.

\section{Preliminary}\label{sec:rs}

In this section, we review (i) randomized smoothing (RS) and its two key limitations, and (ii) the Laplace Bridge (LB), which serves as a foundation for our proposed framework, Laplace-Bridged Smoothing (LBS).

\subsection{Randomized Smoothing (RS)}\label{sec:RS}
RS constructs a certifiably robust classifier from \emph{any} base classifier
$f:\mathbb{R}^d\!\to\!\mathcal{Y}$ by averaging its predictions under random perturbations of the input.
The classical setting uses Gaussian noise \citep{cohen2019certified}, defining the smoothed classifier as
\begin{equation}\label{eq:smooth_g}
    g(x) \;=\; \arg\max_{c\in\mathcal{Y}} \Pr\!\big(f(x+\epsilon)=c\big),
    \epsilon\sim\mathcal{N}(0,\sigma^2 I).
\end{equation}
Let $c_A$ denote the most probable class with probability $p_A$, and $c_B$ the
runner-up class with probability $p_B$.  
Under Gaussian noise, the smoothed classifier $g$ is certifiably constant within
the $\ell_2$ ball of radius $R = \frac{\sigma}{2}\!\left(\Phi^{-1}(p_A) - \Phi^{-1}(p_B)\right)$, where $\Phi^{-1}$ denotes the inverse CDF of the standard Gaussian..
In practice, the exact probabilities $p_A$ and $p_B$ are unknown and must be
estimated from MC algorithms. Therefore, replace $p_A$ with a lower
bound $\underline{p_A}$ and replace $p_B$ with an upper bound
$\overline{p_B}$, yielding a conservative certified radius: 
\begin{equation}
    R=\frac{\sigma}{2}\!\left(\Phi^{-1}(\underline{p_A})-\Phi^{-1}(\overline{p_B})\right),
    \label{eq:cert-radius-general}
\end{equation}

Following \citet{cohen2019certified} and later works~\citep{yang2020randomized,jeong2020consistency,chen2022input},
a standard simplification upper bounds $\overline{p_B}\le 1-\underline{p_A}$,
which reduces \eqref{eq:cert-radius-general} to the two-class radius
$R=\sigma\,\Phi^{-1}(\underline{p_A})$.

Although the smoothed classifier $g$ is robust under noise perturbations, RS presents the following two limitations.

\noindent\textbf{Dependency on Noise-Augmented Training.}
The theoretical guarantee in~\eqref{eq:cert-radius-general} formally holds for any base classifier. However, in practice, its effectiveness critically relies on the stability of the base model $f$ under noise perturbations, which is typically achieved through noise-augmented training \citep{lecuyer2019certified, cohen2019certified, yang2020randomized, chen2022input, voracek2024treatment}. 
Compared to noise-augmented training using the same noise distribution as in certification, if $f$ was trained solely on clean data without noise augmentation,
its predictions can fluctuate substantially under perturbations, and the resulting certified accuracy to degrade sharply, as shown in \Cref{fig:RS} .
Consequently, despite its theoretical generality, RS in practice \emph{requires}
a noise-aware or robustness-oriented base model—typically trained with noise augmentation to produce nontrivial certificates.
This dependency increases training cost, degrades clean accuracy, and limits RS’s practicality as a \emph{post-hoc} defense: while the certification theorem is model-agnostic, RS typically demands changes to the training procedure, preventing direct certification of off-the-shelf models.

\begin{figure}[htbp]
    \centering
    \includegraphics[width=0.48\textwidth]{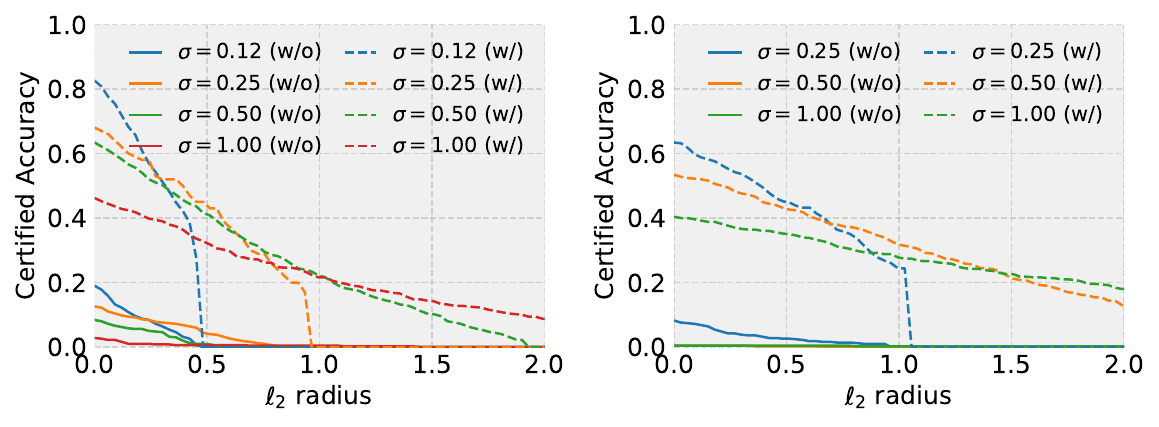}
    \vspace{-7mm}\caption{Certified accuracy of RS~\citep{cohen2019certified} on CIFAR-10 (left) and ImageNet (right) under different Gaussian perturbation levels $\sigma$  
(Solid lines: standard, non-noise-augmented training; dashed lines: standard, noise-augmented training).}
    \label{fig:RS}
\end{figure}


\noindent\textbf{Computational limitation.}
Another major drawback of RS is that, for neural network base
classifiers, the class probabilities cannot be computed in closed
form. Consequently, both the prediction of the smoothed classifier and the certified
radius must be estimated via MC algorithms; since the probabilities
$p_A$ and $p_B$ in~\eqref{eq:cert-radius-general} are intractable expectations
over $f(x+\epsilon)$, accurate certification requires a large number of samples,
leading to substantial computational overhead.

Typical implementations allocate a small pilot budget $n_0$ to identify $c_A$
and a much larger budget $n$ to obtain a one-sided confidence bound for $p_A$, 
commonly the Clopper--Pearson lower bound
${\underline{p_A}}=\mathrm{Beta}^{-1}(\alpha;k_A,n-k_A+1)$ at confidence level $1-\alpha$,
where $k_A$ counts the noisy samples classified of $c_A$.
While statistically rigorous, this sampling step dominates the overall cost:
achieving tight bounds typically requires $n \in [10^4, 10^5]$ forward passes per input,
rendering RS computationally prohibitive for large-scale datasets and, in particular,
for resource-constrained edge devices that lack sufficient parallelism to amortize repeated evaluations.
\textit{In this paper, our aim is to address the two limitations of RS.}

\subsection{The Laplace Bridge (LB)}\label{sec:laplace-bridge}
The LB constructs an analytic correspondence between distributions in logit space and distributions on the probability simplex~\citep{mackay1998choice,hobbhahn2020fast,hennig2012kernel}. 
Recently, it has been applied in Bayesian Neural Networks (BNNs)~\cite{hobbhahn2020fast,daxberger2021laplace} to relate Gaussian and Dirichlet distributions, providing a convenient framework for uncertainty modeling in classification tasks. Let $\pbm=(p_1,\ldots,p_K)^\top\in[0,1]^K$  be a simplex with $\sum_{k=1}^K p_k = 1$ and 
$\bm{\alpha}=(\alpha_1,\ldots,\alpha_K)^\top\in\mathbb{R}_+^K$. 
The Dirichlet distribution has the probability density function (pdf) as follow
\begin{align}\label{eq:pdf_p}
    \mathrm{Dir}_p(\pbm \mid \bm{\alpha})
    = \frac{\Gamma\!\left(\sum_{k=1}^K \alpha_k\right)}
    {\prod_{k=1}^K \Gamma(\alpha_k)}
    \prod_{k=1}^K p_k^{\alpha_k - 1}.
\end{align}
Note that the Dirichlet distribution is defined on a simplex and it can be combined with a softmax function, which is usually applied in classification tasks. 
Let $Z \sim\Ncal(\bm{\mu},\bm{\Sigma})$ be a multivariate Gaussian random variable with mean $\bm{\mu}\in\Rbb^K$ and covariance $\bm{\Sigma}\in S_+^{K}$, where $S_+^{K}$ denotes the set of positive semidefinite matrices in $\Rbb^{K\times K}$. For a realization $\zbm$ of $Z$, the softmax is defined as
$p_k(\zbm)=\exp(z_k)/\sum_{l=1}^K\exp(z_l),\forall 1\leq k\leq K$ and $\zbm$ is also called as a logit of $\pbm$. Using the change of variables induced by the softmax transformation,
the Dirichlet density in \eqref{eq:pdf_p} can be equivalently expressed
in the logit space as
\begin{align}\label{eq:pdf_z}
    p_{\zbm}(\zbm)
    = \frac{\Gamma\!\left(\sum_{k=1}^K \alpha_k\right)}
    {\prod_{k=1}^K \Gamma(\alpha_k)}
    \prod_{k=1}^K p_k(\zbm)^{\alpha_k}.
\end{align}
The reformulation in the basis of logits given by \eqref{eq:pdf_z} offers two benefits: 1) $p_{\zbm}(\zbm)$ remains finite for $|\zbm|\rightarrow\infty$ when $\alpha_k<1$; 2) it is unimodal, with its mode attained at 
$\pbm(\zbm)=\bm{\alpha}/\|\bm{\alpha}\|$. Since $Z$ is Gaussian and its pdf is given in \eqref{eq:pdf_z}, 
there exists a relation between the Dirichlet parameter 
$\bm{\alpha}\in\Rbb_+^K$ and the Gaussian parameters 
$\bm{\mu}\in\Rbb^K,\bm{\Sigma}\in S_+^{K}$. We denote the resulting forward mapping by 
$F_{D\rightarrow G}:\Rbb_+^K\rightarrow\Rbb^K\times S_+^K$, 
and it admits~\cite{hennig2012kernel} that for all $1\leq k,l\leq K$,
\begin{align}
\mu_k & = \log \alpha_k-\frac{1}{K}\sum_{\ell=1}^K \log \alpha_\ell, \label{eq:lb-forward-mean}\\
\Sigmabm_{kl} & = \delta_{kl}\,\frac{1}{\alpha_k}-\frac{1}{K}\!\left(\frac{1}{\alpha_k}+\frac{1}{\alpha_l}
-\frac{1}{K}\sum_{\ell=1}^K\frac{1}{\alpha_\ell}\right). \label{eq:lb-forward-cov}
\end{align}
Since $F_{D\rightarrow G}$ is not injective, we use a pseudo-inverse $F_{G\rightarrow D}^\dagger$ to map Gaussian parameters back to Dirichlet parameters as follows:
\begin{equation}\label{eq:lb-inverse}
\alpha_k =\frac{1}{\Sigmabm_{kk}}\left(1-\frac{2}{K}
+\frac{e^{\mu_k}}{K^2}\sum_{\ell=1}^{K} e^{-\mu_\ell}\right),\forall 1\leq k\leq K,
\end{equation}
which ignores off-diagonal elements of $\bm{\Sigma}$. More discussions see Appendix C in \citet{hobbhahn2020fast}. A further benefit of the approximation~\eqref{eq:lb-inverse} arises from the convenient analytical properties of the Dirichlet exponential family. For example, a point estimate of its posterior is directly given by the Dirichlet’s mean, $\mathbb{E}[p_k]\approx \alpha_k/\alpha_0$ where $\alpha_0=\sum_{k=1}^K\alpha_k$.

\section{Proposed Approach}\label{sec:method}

We propose \emph{Laplace-Bridged Smoothing (LBS)}, a unified framework that integrates the LB into RS to deliver both computational efficiency and statistically grounded $\ell_2$ robustness certification. In contrast to classical RS, which depends heavily on noise-augmented training and extensive high-dimensional sampling, LBS analytically derives a predictive Dirichlet distribution that approximates the smoothed classifier under noise perturbations. This eliminates repeated noisy forward passes, enabling fast and scalable certified robustness estimation across $\ell_2$ norms. The overview of LBS is shown in Figure~\ref{fig:overview}.

\begin{figure}[htbp]
    \centering
    \includegraphics[width=0.48\textwidth]{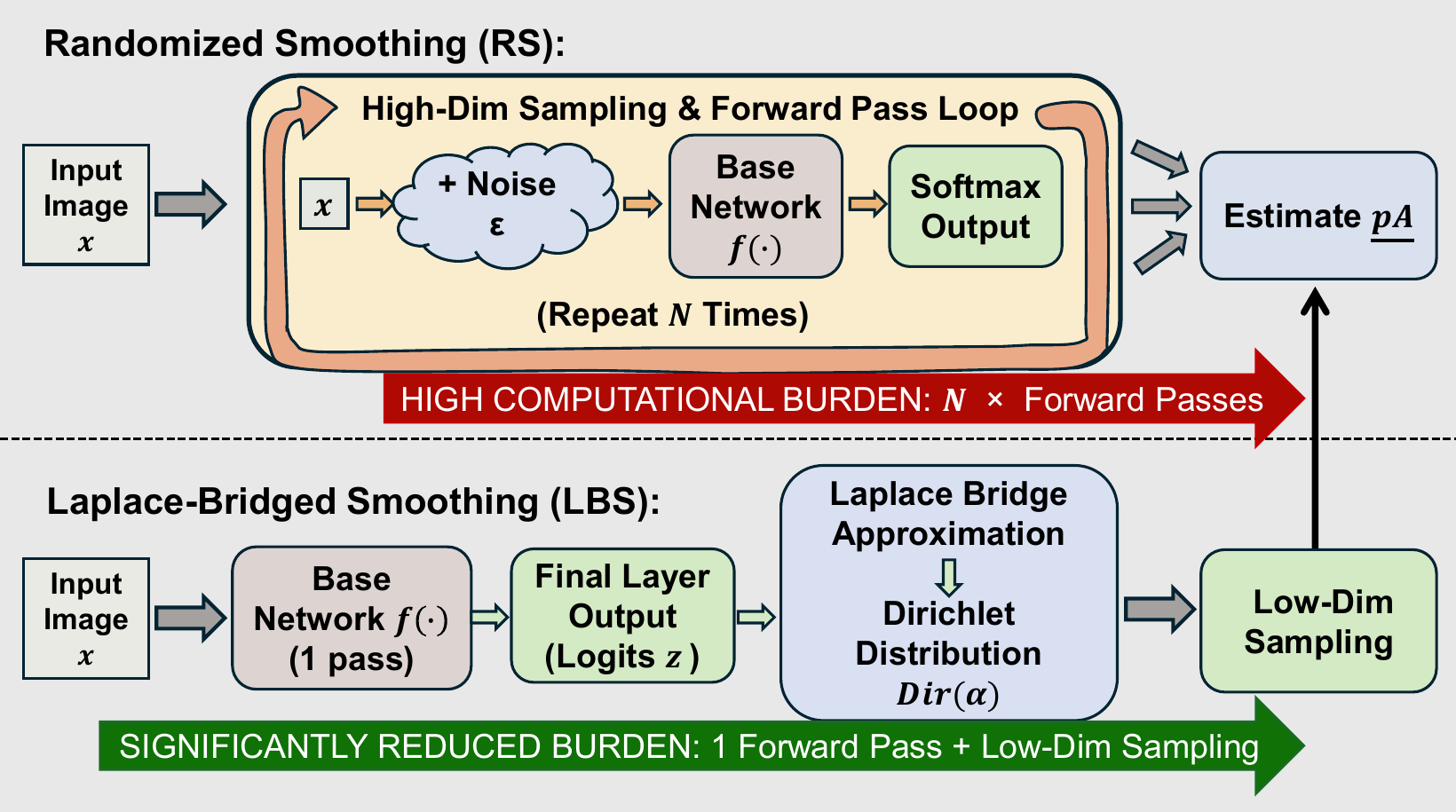}
    \vspace{-7mm}\caption{Overview of RS and LBS.
\textbf{RS}: Certification is performed by repeatedly adding high-dimensional noise to the input and running the base classifier multiple times, relying on repeated noisy forward passes to estimate class probabilities.
\textbf{LBS}: Certification first computes logits from a single forward pass of the base classifier, then applies a LB approximation and performs low-dimensional Dirichlet sampling, which significantly alleviates the computational burden of certification.}
    \label{fig:overview}
\end{figure}

\subsection{Laplace-Bridged Smoothing (LBS)}
Let $f(\cdot|\theta):\Rbb^d\rightarrow\Ycal$ be a neural network with $L$ layers, where $\mathcal{Y}=\{1,\dots,K\}$ denotes the label set with $K$ classes and $\theta$ denote the parameter of all layers. Suppose a feature extractor $\phi(x)\in\Rbb^D$ denote the output of the first $L-1$ layers. The final-layer output is then defined as
$\zbm(x)=W_{\MAP}\phi(x)\in\Rbb^K$ where $W_{\MAP}\in\Rbb^{K\times D}$ is the weight matrix of the last layer. 

Laplace approximation (LA)~\cite{mackay1992practical} can be used to approximate the posterior distribution of the model parameters. Let $\Dcal=\{(x_i,y_i)\}_{i=1}^n$ be the training set and $\Lcal(\Dcal|\theta)$ be the empirical risk. Assume $\theta_{\MAP}$ denote its minimizer, which is also known as a maximum a-posteriori (MAP) estimate. Under the LA, the posterior distribution of $\theta$ is given by   $p(\theta|\Dcal)\approx\Ncal(\theta_{\MAP},\Sigmabm_\theta)$
where $\Sigmabm_\theta = (\nabla_\theta^2\Lcal|_{\theta_{\MAP}})^{-1}$ is the inverse Hessian matrix of $\Lcal$. A practical simplification, used effectively in prior work~\cite{snoek2015scalable, kristiadi2020being}, is to apply the LA only to the final layer of the network rather than to all parameters. In this sense, we focus on the posterior of the last-layer weights $p(\vecc(W) |\Dcal) \approx \mathcal{N}(\vecc(W_{\MAP}), \Sigmabm_W)$, where $W$ denotes a random variable representing the last-layer weights under
the Laplace-approximated posterior, and $\Sigmabm_W$ is derived analogously to
$\Sigmabm_\theta$ as the inverse empirical Hessian with respect to $W$.

In the traditional RS, the smoothed classifier in \eqref{eq:smooth_g} is computed by drawing many samples of $f(x+\epsilon_i)$ and selecting the class that appears most frequently. This sampling procedure is necessary because the distribution of $f(x+\epsilon)$ is not known in closed form. However, under a Bayesian framework, we can approximate the behavior of the final-layer output directly. For a fixed perturbation $\epsilon$ with small $\|\epsilon\|$, we apply a first-order Taylor expansion to the feature extractor:
\[
\phi(x+\epsilon) \approx \phi(x) + J_\phi(x)\epsilon,
\]
where $J_\phi(x)$ is the Jacobian of $\phi$ with respect to $x$. It is natural to assume that $\epsilon$ and $W$ are independent and the logits $Z$ can be approximated as a Gaussian. The result is given in Proposition~\ref{prop:approx_z} and its proof can be found in the Appendix~\ref{proof:proposition-approx_z}. 


\begin{proposition}\label{prop:approx_z}
Let $\epsilon$ be a noise vector with zero mean and covariance $\Sigmabm_\epsilon$, independent of  $W\sim\Ncal(\mubm_W,\Sigmabm_W)$ where $\mubm_W := W_{\MAP}\in\Rbb^{K\times D}$ denotes the MAP estimate of $W$. Then the random logits $Z$ are approximately distributed as a Gaussian, $Z \sim \mathcal N(\boldsymbol\mu_z, \boldsymbol\Sigma_z)$,
where $\boldsymbol\mu_z$ and $\boldsymbol\Sigma_z$ are given by
\begin{align}
\bm{\mu}_z & = \mubm_W\phi(x), \nonumber\\
\boldsymbol\Sigma_z&= (I_K\!\otimes\!\phi^\top)\Sigmabm_W(I_K\!\otimes\!\phi^\top)^\top \nonumber\\
& + \mubm_{W}J_\phi\Sigmabm_\epsilon J_\phi^\top \mubm_W^\top
+ \operatorname{Tr}_D\!\big[(I_K\!\otimes\!J_\phi\Sigmabm_\epsilon J_\phi^\top)\Sigmabm_W\big],\label{eq:logit-cov}
\end{align}
and $\operatorname{Tr}_D : \mathbb{R}^{KD\times KD} \to \mathbb{R}^{K\times K}$ denotes the operator that maps a block matrix to the matrix whose $(i,j)$-th entry is the trace of its $D\times D$ block $\Sigmabm_{ij}$.
\end{proposition}

This Gaussian approximation captures the first-order propagation of both input perturbations and posterior curvature in parameter space, yielding a tractable characterization of logit variability under generic noise smoothing. Applying the LB approximation to $\pbm=\mathrm{softmax}(Z)$ yields a Dirichlet distribution $\pbm\sim \mathrm{Dir}_p(\pbm|\bm{\alpha})$, with parameters determined by \eqref{eq:lb-inverse} and using $\mubm=\mubm_z$ and $\Sigmabm=\Sigmabm_z$ of Proposition~\ref{prop:approx_z}.

This Dirichlet distribution acts as a proxy for the smoothed prediction distribution, with predictive $\mathbb{E}[p_k] = \alpha_k / \alpha_0$ and $\mathrm{Var}[p_k] = \alpha_k(\alpha_0 - \alpha_k)/[\alpha_0^2(\alpha_0 + 1)]$, here $\alpha_0 = \sum_{k=1}^{K} \alpha_k$. 
The predicted class is taken as
$c_A = \arg\max_k \alpha_k$, corresponding to the mode of the Dirichlet distribution.

\paragraph{Certification.}
Given the Dirichlet parameters $\alphabm$, we treat the distribution $\pbm\sim\mathrm{Dir}_p(\pbm|\bm{\alpha})$ as an analytic surrogate for the smoothed predictive distribution. 
To estimate the top-class probability, we draw $N$ MC samples $\pibm^{(1)}, \ldots, \pibm^{(N)} \sim \mathrm{Dir}_p(\pbm|\bm{\alpha})$ and estimate
\begin{equation}
p_A = \frac{1}{N} \sum_{i=1}^N \mathds{1}\left\{\arg\max_j \pi_j^{(i)} = c_A\right\}.
\end{equation}
A one-sided Binomial proportion confidence interval (e.g., Clopper--Pearson)
is then applied at level $1-\eta$ to obtain a statistically valid lower bound
\begin{align}\label{eq:pA_LBS}
    {\underline{p_A}}^{\text{(LBS)}}=\mathrm{Beta}^{-1}\left(\tfrac{\eta}{2}; n_A, N - n_A + 1\right),
\end{align}
Here, $\eta$ is the confidence parameter, so that the lower bound 
$\underline{p_A}^{(\mathrm{LBS})}$ holds with probability at least $1-\eta$,
where $\mathrm{Beta}^{-1}(q;a,b)$ is the $q$-th quantile of a Beta distribution $\mathrm{Beta}(a,b)$ and $n_A = \sum_{i=1}^N \mathds{1}\!\left\{\arg\max_j \pi_j^{(i)} = c_A\right\}$.

This procedure in \eqref{eq:pA_LBS} provides a lightweight estimate of the lower bound on the top-class probability $p_A$. Instead of estimating $p_A$ through repeated high-dimensional noisy forward passes of the base classifier, LBS constructs a low-dimensional probabilistic surrogate and samples from a $K$-dimensional Dirichlet distribution. Therefore, LBS preserves the same certification interface as standard RS~\citep{cohen2019certified}: once an LBS-based lower bound $\underline{p_A}^{(\mathrm{LBS})}$ is obtained, it can be plugged into the corresponding radius function $R(\cdot)$.

Under $\ell_2$ certification with Gaussian noise $\epsilon \sim \mathcal{N}(0,\sigma^2 I)$, LBS computes the certified radius:
\[
R\!\left(\underline{p_A}^{(\mathrm{LBS})}\right)
=
\sigma\,\Phi^{-1}\!\left(\underline{p_A}^{(\mathrm{LBS})}\right).
\]
where $\Phi^{-1}$ is the standard normal quantile function~\citep{cohen2019certified}.

\begin{algorithm}[t]
\caption{LBS Certification}
\label{alg:lbs}
\textbf{Require:} input $x$; feature extractor $\phi(\cdot)$; last-layer posterior $(W_{\mathrm{MAP}},\Sigmabm_W)$; noise variance $\Sigmabm_\epsilon$; Dirichlet sample size $N$; significance level $\eta$ \\
\textbf{Ensure:} predicted class $c_A$ and certified radius $R({\underline{p_A}}^{\text{(LBS)}})$

\begin{algorithmic}[1]
\STATE $\phi \leftarrow \phi(x)$; $J_\phi(x) \leftarrow \frac{\partial \phi(x)}{\partial x}$
\STATE $\mubm_z \leftarrow W_{\mathrm{MAP}}\,\phi$ 
\STATE $\Sigmabm_z \leftarrow (I_K \otimes \phi^\top)\,\Sigma_W\,(I_K \otimes \phi^\top)^\top
+W_{\mathrm{MAP}}J_\phi\Sigmabm_\epsilon J_\phi^\top W_{\mathrm{MAP}}^\top+ \operatorname{Tr}_D\!\big[(I_K\!\otimes\!J_\phi\Sigmabm_\epsilon J_\phi^\top)\Sigmabm_W\big]$
\FOR{$k=1,\ldots,K$}
    \STATE $\alpha_k \gets \frac{1}{[\Sigmabm_z]_{kk}}\!\left(1-\frac{2}{K}+\frac{e^{\mu_{z,k}}}{K^2}\sum_{\ell=1}^{K} e^{-\mu_{z,\ell}}\right)$
\ENDFOR
\STATE 
$c_A \gets \argmax_k \alpha_k$
\STATE Sample $\pi^{(1)},\ldots,\pi^{(N)} \sim \mathrm{Dir}_p(\pbm|\alphabm)$
\STATE $n_A \leftarrow \sum_{i=1}^N \mathds{1}\!\left[\arg\max_j \pi^{(i)}_j = c_A\right]$
\STATE $\underline{p_A} \leftarrow \mathrm{Beta}^{-1}\left(\tfrac{\eta}{2}; n_A, N - n_A + 1\right)$
\IF{$\underline{p_A} \le 0.5$}
    \STATE \textbf{return} \textsc{Abstain}
\ELSE
    \STATE \textbf{return} $(c_A, R({\underline{p_A}}^{\text{(LBS)}}))$
\ENDIF
\end{algorithmic}
\end{algorithm}




\subsection{Theoretical analysis}

Our proposed LBS framework applies LB within RS. In this section, we perform the error analysis of the LB approximation. 
One important aspect of the LBS framework is that it draws samples from the Dirichlet distribution rather than Gaussian distributions, which avoids repeatedly performing forward passes through the model. Let the logistic–normal predictive distribution be defined as
$\pibm=\softmax(Z)$ with $Z\sim\mathcal N(\mubm_z,\Sigmabm_z)$ and let its Dirichlet surrogate be denoted by
$\tilde\pibm\sim\mathrm{Dir}_p(\pbm| \bm{\alpha})$. 
\Cref{prop:lb_kl} establishes that $\pibm$ and $\tilde{\pibm}$ are close in distribution. In particular, the Kullback–Leibler (KL) divergence between $\pibm$ and $\tilde{\pibm}$ is bounded from above.

\begin{theorem}[Local KL bound for the LB]
\label{prop:lb_kl}
Let $\pibm=\softmax(Z)$ with $Z\sim\mathcal N(\mubm_z,\Sigmabm_z)$ and $\tilde\pibm\sim\mathrm{Dir}_p(\pbm| \bm{\alpha})$ be its Dirichlet surrogate where $\alphabm$ is given by \eqref{eq:lb-inverse}. Assume that:
\begin{enumerate}[nosep,label=(\roman*)]
\item $\|\Sigmabm_z\|_F \leq \varepsilon$ for some small $\varepsilon>0$;
\item the Jacobian and Hessian of the softmax mapping evaluated at $\mubm_z$
are uniformly bounded;
\item both densities $P_{\pibm}$ and $P_{\tilde{\pibm}}$ are twice continuously differentiable with uniformly bounded local curvature in a neighborhood of $\mathbb{E}[\pibm]$ inside the simplex $\Delta^{K-1}$. 
\end{enumerate}
Under assumptions (i)-(iii), there exists a constant $C(\mubm_z)>0$, independent of $\Sigma_z$, such that 
\begin{equation}
\label{eq:kl-local-bound}
\KL(P_{\pibm}\|P_{\tilde{\pibm}})\leq C(\mubm_z)\|\Sigmabm_z\|_F+\mathcal O(\|\Sigmabm_z\|_F^{\frac{3}{2}}).
\end{equation}
Moreover, $P_{\pibm} \rightarrow P_{\tilde{\pibm}}$ as
$\|\Sigma_z\|_F\to 0$ in the sense of local asymptotic normality
\citep{van2000asymptotic}.

\end{theorem}

The bound in \Cref{prop:lb_kl} shows that the LB approximation error scales linearly with the logit
covariance $\|\Sigmabm_z\|_F$ to first order. Consequently, the LB surrogate is accurate in regimes where the posterior over $W$ is concentrated or where the softmax map is locally flat.These are precisely the regimes that dominate robustness certification. Therefore, assumption (i) in \Cref{prop:lb_kl} is standard in Laplace-based analyses (e.g., \citealt{mackay1992practical,van2000asymptotic}). The proof is deferred to Appendix~\ref{proof:proposition4.2}.

\subsection{Computational efficiency}
Using LBS, we can analytically approximate the density of the softmax-Gaussian random variable produced by the Gaussian weight and injected noise as a Dirichlet distribution, rather than relying on a large number of samples. As shown in \eqref{eq:lb-inverse}, constructing the $K$ Dirichlet parameters $\alpha_K$ requires $\Ocal(K)$ computations, and sampling from the resulting distribution costs an additional $\Ocal(N)$, yielding a total complexity of $\Ocal(K+N)$. In contrast, RS requires generating the set ${f(x+\epsilon_i)}$ and then sampling, which costs $\Ocal(n C_f + N)$, where $n$ is the number of noisy samples and $C_f$ is the cost of a single forward pass. Typically, $C_f \gg N$ and $n C_f \gg K$, meaning that RS spends most of its computational resources generating noisy predictions, whereas the proposed LBS framework avoids this overhead entirely.

\section{Experiments}\label{sec:experiments}



\noindent \textbf{Datasets.}
We conduct experiments on CIFAR-10 \citep{krizhevsky2009learning} and ImageNet \citep{deng2009imagenet} benchmark datasets commonly used in certified robustness. 

\noindent \textbf{Training.}
For CIFAR-10, we employ a ResNet-110 classifier trained following the procedure of \citet{cohen2019certified}, except that we do \emph{not} apply noise augmentation during training. For ImageNet, we adopt a ViT-B/16 \citep{dosovitskiy2020image} model trained under the same data augmentation and optimization scheme, also without noise injection in the base training process. The results for different architectures (DenseNet-161 and EfficientNet-B0) are shown in Appendix~\ref{model_architectures}. 

To demonstrate the scalability of our approach across diverse computational environments, all experiments are conducted on NVIDIA A100-SXM4-80GB GPUs as well as resource-constrained edge devices including a NVIDIA Jetson Orin Nano and a Raspberry~Pi~4. The real on-device deployment setup and implementation details are provided in Appendix~\ref{edge_app}, with results discussed in Section~\ref{edge}.

\paragraph{Evaluation metrics.}
We consider two metrics: (a) certified accuracy at radius $R$, defined as the fraction of correctly classified test samples that are certifiably robust within radius $R$; and (b) certification time, measured as the average end-to-end wall-clock latency per sample.  


\begin{figure}[htbp]
    \centering
    \includegraphics[width=0.48\textwidth]{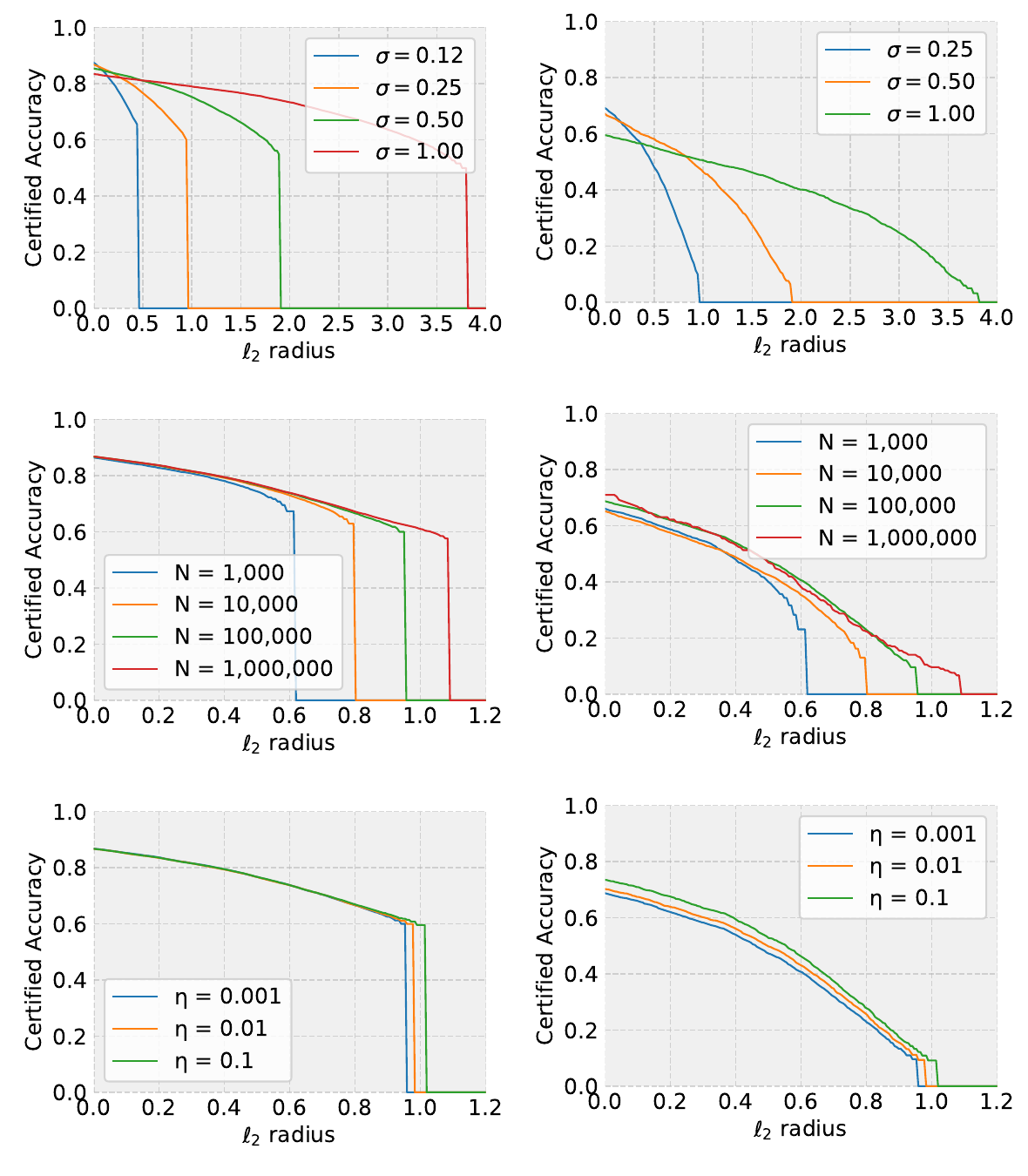}
    \vspace{-7mm}\caption{Certified accuracy of LBS on CIFAR10 (left) and ImageNet (right) under different perturbation levels $\sigma$ with $N$ =100,000, $\eta$ = 0.001 (top), various Dirichlet sample size $N$ with $\sigma$ = 0.25, $\eta$ = 0.001 (middle) and confidence levels $\eta$ with $\sigma$ = 0.25, $N$ =100,000 (bottom).}
    \label{fig:LBS}
\end{figure}
\subsection{LBS Certification}\label{performance_LBS}
Figure~\ref{fig:LBS} reports the certified accuracy of LBS on CIFAR-10 (left) and ImageNet (right) while varying the noise level $\sigma$ (top), the Dirichlet sample size $N$ (middle), and the confidence level $\eta$ (bottom) used in the one-sided bound.
As shown, $\sigma$ governs the robustness--accuracy trade-off: smaller $\sigma$ yields higher clean accuracy but smaller certified radii, whereas larger $\sigma$ improves certified robustness at some accuracy cost.
Increasing $N$ yields modest yet consistent gains in certified radii, while decreasing $\eta$ produces slightly more conservative certificates.
Overall, these results demonstrate that LBS exhibits stable and
predictable certification behavior across a wide range of hyperparameter
settings.
Although LBS operates effectively without noise augmentation, we also examine its behavior when noise augmentation is applied. The results are shown in Appendix~\ref{LBS_with_noise_augmentation}.

\subsection{Efficiency and Robustness}
Prior work has sought to reduce the computational cost of RS from different perspectives. To evaluate both robustness and efficiency, we compare LBS with standard RS~\cite{cohen2019certified}, ISS~\cite{chen2022input}, ADRE~\cite{feng2020regularized}, Betting CS~\cite{voracek2024treatment}, Accelerated Smoothing (AS)~\cite{bhardwaj2024accelerated}, ensemble-based Gaussian smoothing~\cite{horvath2022boosting}, and UCAN~\cite{hong2025towards}. \Cref{tab:efficiency} reports certified accuracy at multiple $\ell_2$ radii together with the per-sample certification time.

The results show that LBS achieves a strong overall robustness--efficiency trade-off on both CIFAR-10 and ImageNet. In terms of certified robustness, LBS obtains the best certified accuracy on CIFAR-10 across all reported radii, and achieves the strongest or highly competitive performance on ImageNet. At the same time, LBS substantially reduces certification time, requiring only $0.7$s per sample on CIFAR-10 and $10.8$s on ImageNet, compared with $9.8$s and over $120$s for standard sampling-based RS methods.

Although AS also achieves fast inference-time certification, reporting $0.4$s on CIFAR-10 and $6.2$s on ImageNet, it relies on training an additional surrogate model and offline target generation, with reported costs of up to $4200$s for surrogate training and 180--190 hours of target generation. More broadly, existing acceleration and robustness-oriented methods still depend on repeated sampling from the smoothing distribution or additional training components. In contrast, LBS reduces certification cost by replacing repeated input-space noisy forward passes with a low-dimensional Laplace-Bridge surrogate, achieving strong certified robustness with much lower certification time.

\begin{table}[h]
\centering
\resizebox{0.48\textwidth}{!}{
\begin{tabular}{c|lcccccc}
\toprule
 & $\ell_2$ radius  & 0.0 & 0.12 & 0.25 & 0.5 & 1.0 & Time(s) \\
\midrule
\multirow{14}{*}{\rotatebox{90}{CIFAR-10}}
 & RS (w/o)  & 8.4&6.2 & 4.8 & 0.8&0.0& 9.8 \\
 & ADRE (w/o) & 12.6 & 10.0 & 8.4 & 4.2 & 0.0 & 9.6\\
 & ISS (w/o) &6.8& 4.4 & 0.8 & 0.4 & 0.0 & 6.6\\
  & RS (w/)  & 63.4&58.4 & 51.8 & 41.4&21.8& 9.8 \\
 & ADRE (w/) & 65.2 & 61.0 & 56.4 & 46.7 & 28.1 & 9.6\\
 & ISS (w/) &61.2& 56.6 & 49.9 & 39.3 & 20.0 & 6.6\\
& Betting CS (w/o)  & 9.1&8.0 & 6.2 & 4.7&0.0& 1.2 \\
& AS (w/o)  & 13.0&11.6 & 9.8 & 6.0&0.0& 0.4 \\
 & UCAN (w/o) & 10.8&9.2 &7.6 &4.0&0.0& 9.4 \\
& Betting CS (w/)  & 63.6&58.8 & 51.2 &40.9&22.0& 1.2 \\
& AS (w/)  & 66.4&63.8 & 59.3 &48.1&30.4& 0.4 \\
& Ensemble (w/) & 68.8&64.7 &60.1 &48.4&29.5& 9.2 \\
& UCAN (w/) & 83.3&81.7 &80.6&78.4&74.0& 9.4 \\
 & \textbf{LBS} & \textbf{85.4} & \textbf{84.4} & \textbf{83.4} & \textbf{81.1} & \textbf{75.2} & \textbf{0.7} \\
\midrule
\multirow{14}{*}{\rotatebox{90}{ImageNet}}
 & RS (w/o)  & 0.4&0.4 & 0.4 & 0.2&0.0& 123.5 \\
 & ADRE (w/o) & 4.2 & 2.8 & 2.2 & 0.6 & 0.0 & 122.7\\
 & ISS (w/o) & 0.4&0.4 & 0.4 & 0.2&0.0 & 109.9\\
 & RS (w/)  &56.5& 55.4 & 52.7 & 45.4 & 34.6 & 123.5\\
 & ADRE (w/) &60.1& 57.5 & 55.3 & 49.9 & 38.8& 122.8 \\
 & ISS (w/) & 53.6 & 50.5 & 48.7 & 41.0 & 30.4& 109.9 \\
 & Betting CS (w/o)  & 0.6&0.6 & 0.4 & 0.4&0.0& 24.9 \\
 & AS (w/o)  & 4.0&3.2 & 2.6 & 0.6&0.0& 6.2 \\
 & UCAN (w/o) & 4.0&3.2 &2.4 &0.6&0.0& 121.6 \\
& Betting CS (w/)  & 56.4&55.4 &53.0 &45.6&34.2& 24.9 \\
& AS (w/)  & 58.4&57.0 & 53.9 &48.1&34.8& 6.2 \\
& Ensemble (w/) & 61.2&59.8 &56.5 &50.1&39.6& 120.3 \\
& UCAN (w/) & 65.3&63.7 &62.0 &\textbf{58.9}&\textbf{50.1}& 121.6 \\
 & \textbf{LBS} & \textbf{67.0} & \textbf{64.9} & \textbf{62.5} & 58.2 & 46.5 & \textbf{10.8} \\
\bottomrule
\end{tabular}}
\caption{Certified accuracy and certification time of LBS versus prior state-of-the-art methods at different $\ell_2$ radii with $\sigma$ = 0.5 and $N$ =100,000 on CIFAR-10 and ImageNet (w/o: standard, non-noise-augmented training; w/: standard, noise-augmented training).}
\label{tab:efficiency}
\end{table}

\subsection{On-device LBS Certification}\label{edge}

We evaluate the efficiency and practicality of LBS under on-device computational constraints, which provide a stringent testbed for scalable certified robustness.
Compared to server-class hardware, on-device platforms operate under limited compute, memory, and power budgets, making the high MC cost of conventional RS particularly prohibitive. 

We evaluate on-device certification on two representative platforms: NVIDIA Jetson Orin Nano and Raspberry~Pi~4. 
On both devices, we deploy ResNet110, SqueezeNet, and ShuffleNet classifiers trained on CIFAR-10 without noise augmentation following~\citet{cohen2019certified}.
Figure~\ref{fig:LBS_edgedevices_sigma} reports the certified accuracy of LBS under varying perturbation levels $\sigma$.
Across both platforms, LBS exhibits consistent certification trends that mirror those on high-end GPUs: smaller $\sigma$ yields higher clean accuracy but smaller certified radii, while larger $\sigma$ improves certified robustness at the cost of accuracy.


Table~\ref{tab:certfication_time_edge} reports the per-sample certification time of RS and LBS under multiple CIFAR-10 model architectures.
On the Jetson Orin Nano, LBS consistently outperforms RS, reducing the certification time from 91.8,s to 2.8,s for ResNet110 and from 33.3,s to 1.2,s for SqueezeNet.
On the Raspberry~Pi~4, the efficiency gains are even more pronounced: LBS accelerates certification from 690.9,s to 1.4,s for SqueezeNet and from 240.2,s to 0.9,s for ShuffleNet. These results show that LBS dramatically reduces on-device certification latency, enabling practical and scalable robustness guarantees on resource-constrained edge hardware. Additional on-device results, including certified accuracy trends under varying sample sizes~$N$, are provided in the Appendix~\ref{edge_app}.

\begin{figure}[htbp]
    \centering
    \includegraphics[width=0.48\textwidth]{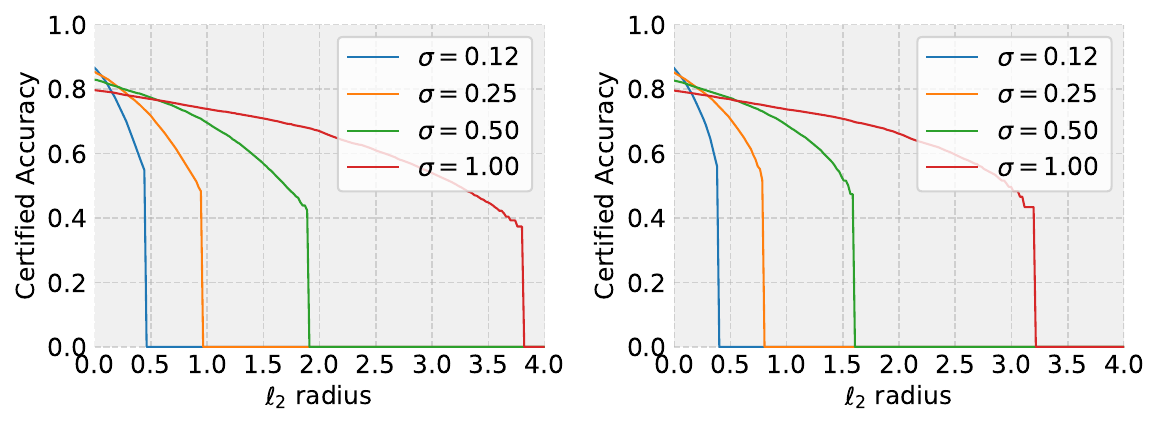}
    \vspace{-7mm}\caption{Certified accuracy of LBS on the NVIDIA Jetson Orin Nano (left) and the Raspberry~Pi~4 (right) using a SqueezeNet-based classifier on CIFAR-10 under different perturbation levels $\sigma$, with $N=100{,}000$ and $\eta=0.001$.}
    \label{fig:LBS_edgedevices_sigma}
\end{figure}

\begin{table}[h]
\centering
\resizebox{0.48\textwidth}{!}{
\begin{tabular}{lcccc}
\toprule
 & \multicolumn{2}{c}{Jetson Orin Nano} & \multicolumn{2}{c}{Raspberry~Pi~4} \\
\cmidrule(lr){2-3}\cmidrule(lr){4-5}
 & ResNet110 & SqueezeNet & SqueezeNet& ShuffleNet \\
\midrule
RS      & 91.8  & 33.3  & 690.9 &240.2\\
LBS     & 2.8   & 1.2   & 1.4   & 0.9 \\
\textbf{Speedup} & \textbf{33$\times$} & \textbf{28$\times$} & \textbf{494$\times$} & \textbf{267$\times$}\\
\bottomrule
\end{tabular}}
\caption{Certification time (s) per sample for RS and LBS with $N$ =100,000 on CIFAR-10 across different model architectures, together with the resulting speedup of LBS over RS.}
\label{tab:certfication_time_edge}
\end{table}

\section{Discussion}\label{sec:discussion}

Our current theoretical development (Proposition~\ref{prop:approx_z}) is derived under isotropic Gaussian noise, which directly supports certification in the $\ell_2$ setting.

To illustrate this practical generality, we additionally evaluated LBS under $\ell_1$ and $\ell_\infty$ certification following~\citet{yang2020randomized} to compute the corresponding radii. 



Table~\ref{tab:l1_baselines} reports $\ell_1$ certified accuracy across radii and compares LBS with prior $\ell_1$ RS baselines under Laplace and Uniform noise.
Without noise-augmented training, all baselines collapse to very low certified accuracy on CIFAR-10 and ImageNet (below 15\% and 7\%, respectively), indicating a strong reliance on noise augmentation.
Even with noise augmentation, LBS consistently outperforms the strongest baselines across all $\ell_1$ radii on both datasets, demonstrating that LBS substantially reduces dependence on noise-augmented training while preserving strong certified robustness.

Appendix~\ref{l_infinity} further evaluates LBS under an $\ell_\infty$ threat model, demonstrating stable and consistent certification behavior.
We also validate the robustness of LBS under both $\ell_1$ and $\ell_\infty$ perturbations in on-device settings (Appendix~\ref{edge_app}), confirming its suitability for deployment on resource-constrained platforms.
Together, these results show that LBS is a deployment-ready certification framework with strong performance across $\ell_p$ threat models.
Extending the theoretical analysis of LBS to more general noise distributions remains an important direction for future work.

\begin{table}[h]
\centering
\resizebox{0.48\textwidth}{!}{
\begin{tabular}{c|lccccccc}
\toprule
& $\ell_1$ radius $\rightarrow$ & 0.5 & 1.0 & 1.5 & 2.0 & 2.5 & 3.0 \\
\midrule
\multirow{8}{*}{\rotatebox{90}{CIFAR-10}}
& Laplace (w/o), \citet{lecuyer2019certified} & 10.6 & 10.0 & 9.4 & 8.3 & 7.8 & 7.1 \\
& Laplace (w/o), \citet{tengl1}  & 12.8 & 12.3 & 12.2 & 11.2 & 10.4 & 10.2 \\
& Uniform (w/o), \citet{yang2020randomized}  & 10.2 & 10.4 & 9.2 & 10.6 & 10.1 & 9.1 \\
& Laplace (w/), \citet{lecuyer2019certified} & 55.7 & 33.2 & 21.6 & 17.0 & 10.5 & 6.4 \\
& Laplace (w/), \citet{tengl1}  & 60.4 & 38.8 & 24.1 & 16.3 & 10.7 & 7.1 \\
& Uniform (w/), \citet{yang2020randomized}  & 69.8 & 59.4 & 50.3 & 43.7 & 31.9 & 29.2 \\
& \textbf{Laplace, LBS} & \textbf{84.3} & \textbf{82.0} & \textbf{80.2} & \textbf{78.4} & \textbf{77.0} & \textbf{76.1}  \\
& \textbf{Uniform, LBS} & \textbf{86.1} & \textbf{85.0} & \textbf{83.8} & \textbf{82.7} & \textbf{81.9} & \textbf{81.0} \\
\midrule
\multirow{8}{*}{\rotatebox{90}{ImageNet}}
& Laplace (w/o), \citet{lecuyer2019certified} &  5.4 & 3.9 & 3.0 & 2.2 & 1.8 & 1.0 \\
& Laplace (w/o), \citet{tengl1} & 5.8 & 4.2 & 2.9 & 2.1 & 1.4 & 0.7 \\
& Uniform (w/o), \citet{yang2020randomized} & 5.7 & 2.7 & 2.0 & 1.6 & 1.1 & 0.7 \\
& Laplace (w/), \citet{lecuyer2019certified} &  42.8 & 39.5 & 34.2 & 22.8 & 18.9 & 10.4 \\
& Laplace (w/), \citet{tengl1} & 47.1 & 39.7 & 31.6 & 26.5 & 23.0 & 18.2 \\
& Uniform (w/), \citet{yang2020randomized} & 53.6 & 48.2 & 42.5 & 37.7 & 35.1 & 30.8 \\
& \textbf{Laplace, LBS} & \textbf{70.5} & \textbf{63.1} & \textbf{57.0} & \textbf{52.8} & \textbf{50.7} & \textbf{48.1} \\ 
& \textbf{Uniform, LBS} & \textbf{71.6} & \textbf{66.3} & \textbf{61.4} & \textbf{54.6} & \textbf{52.5} & \textbf{49.9} \\
\bottomrule
\end{tabular}}
\caption{Certified accuracy of LBS vs previous state-of-the-art, at various $\ell_1$ radius on CIFAR-10 and ImageNet datasets (w/o: standard, non-noise-augmented training; w/: standard, noise-augmented training).}
\label{tab:l1_baselines}
\end{table}

\section{Conclusion}\label{sec:conclusion}
We proposed \emph{Laplace-Bridged Smoothing (LBS)}, an analytic framework that simultaneously improves certified accuracy and certification efficiency over conventional RS.
By propagating input noise through a locally linearized model and leveraging a LB to construct a Dirichlet surrogate, LBS enables statistically sound robustness certification with substantially reduced computational cost.
Across CIFAR-10 and ImageNet, LBS achieves higher certified accuracy than standard RS while accelerating certification by nearly an order of magnitude.
On resource-constrained edge devices such as Jetson Orin Nano and Raspberry~Pi~4, LBS achieves speedups of up to $494\times$, benefiting from avoiding high-dimensional MC sampling that cannot be efficiently batched or parallelized under limited compute and memory, demonstrating its practicality for real-world deployment.

\section{Impact Statement}
This paper presents work whose goal is to advance the field of Machine Learning. There are many potential societal consequences of our work, none which we feel must be specifically highlighted here.

\bibliography{ref}
\bibliographystyle{icml2026}

\newpage
\appendix
\onecolumn

\subsection*{Appendix Contents}
\addcontentsline{toc}{subsection}{Appendix Contents}
\vskip3pt\hrule\vskip5pt

\begin{itemize}
    \setlength\itemsep{0.8em}

    \item \hyperref[proof:proposition-approx_z]{\textbf{A.1} Proof of Proposition 4.1}
    \item \hyperref[proof:proposition4.2]{\textbf{A.2} Proof of Theorem 4.2}
    \item \hyperref[l_infinity]{\textbf{A.3} LBS Certification under $\ell_\infty$ Perturbations}
    \item \hyperref[model_architectures]{\textbf{A.4} Effect of Different Model Architectures}
    \item \hyperref[LBS_with_noise_augmentation]{\textbf{A.5} LBS Certification with Noise Augmentation}
    \item \hyperref[edge_app]{\textbf{A.6} On-device LBS Certification}
    \item \hyperref[llms]{\textbf{A.8} Use of Large Language Models (LLMs)}
\end{itemize}

\vskip3pt\hrule\vskip5pt
\clearpage
\section{Appendix}\label{sec:appendix}

\subsection{Proof of Proposition~\ref{prop:approx_z}}\label{proof:proposition-approx_z}

\begin{proof}
For a fixed perturbation $\epsilon$ with sufficiently small norm $\|\epsilon\|$, a first-order Taylor approximation yields
$\phi(x+\epsilon)\approx \phi(x)+J_\phi(x)\epsilon$, 
where $J_\phi(x)=\partial \phi(x)/\partial x$ is the Jacobian matrix. Consequently, $\zbm(x+\epsilon)=W\phi(x+\epsilon)\approx W\big(\phi(x)+J_\phi(x)\epsilon\big)$.

We model the uncertainty of the last-layer weights $W$ using the Laplace approximation (LA) \cite{snoek2015scalable,kristiadi2020being}. Specifically,
\[
\operatorname{vec}(W)\sim\mathcal{N}\!\big(\operatorname{vec}(W_{\mathrm{MAP}}),\Sigmabm_W\big),
\]
where $\operatorname{vec}(W)\in\mathbb{R}^{KD}$ denotes the vectorization of $W$, and $\Sigmabm_W := \Big(-\nabla^2\big|_{W=W_{\mathrm{MAP}}}\log p(W|\mathcal{D})\Big)^{-1}
\in\mathbb{R}^{KD\times KD}$ 
is the inverse Hessian of the log-posterior evaluated at the MAP estimate $W_{\mathrm{MAP}}$.

Let $w_i\in\mathbb{R}^D$ denote the $i$-th row of $W$, so that $W=[w_1,\ldots,w_K]^\top\in\mathbb{R}^{K\times D}$. 
We partition the covariance matrix $\Sigmabm_W$ into blocks, $\Sigmabm_W=
\begin{bmatrix}
\Sigmabm_{11} & \cdots & \Sigmabm_{1K}\\
\vdots & \ddots & \vdots\\
\Sigmabm_{K1} & \cdots & \Sigmabm_{KK}
\end{bmatrix},$
where each block $\Sigmabm_{ij}\in\mathbb{R}^{D\times D}$. It follows that each row satisfies $w_i\sim\mathcal{N}(\mu_i,\Sigmabm_{ii})$, where $\mu_i$ is the $i$-th row of $W_{\mathrm{MAP}}$.

Define $U := W\big(\phi(x)+J_\phi(x)\epsilon\big)\in\mathbb{R}^K$.
Each component $u_i = w_i^\top(\phi+J_\phi\epsilon)$ is therefore a bilinear form involving one Gaussian random vector and one general random vector. 
The marginal distribution $p(u)=\int p(u| W)p(W)\,dW$
is generally non-Gaussian and does not admit a closed-form expression.
A standard and tractable approach is to approximate $U$ by a moment-matched Gaussian distribution \cite{hernandez2015probabilistic,ghosh2016assumed}, $q(u)=\mathcal{N}(\mubm_u,\Sigmabm_u)$, 
where $\mubm_u=\mathbb{E}[U]$ and $\Sigmabm_u=\operatorname{Cov}(U)$.

The mean is given by $\mathbb{E}[U]=\mathbb{E}[\mathbb{E}[U| W]]=W_{\mathrm{MAP}}\phi$. 
The covariance follows from the law of total variance:
\begin{align}\label{eq:cov_U}
\operatorname{Cov}(U)
&= \operatorname{Cov}\big(\mathbb{E}[U| W]\big)
+ \mathbb{E}\big[\operatorname{Cov}(U| W)\big] = \operatorname{Cov}(W\phi)
+ \mathbb{E}\big[W J_\phi\Sigmabm_\epsilon J_\phi^\top W^\top\big].
\end{align}
Using the mixed-product property of matrix product and Kronecker product $\vecc(BVA^\top)=(A\otimes B)\vecc(V)$, we can rewrite $W\phi=(I_K\otimes \phi^\top)\vecc(W)$
which leads to 
\begin{align}\label{eq:cov_U1}
    \operatorname{Cov}(W\phi)= (I_K\otimes \phi^\top)\Sigmabm_W(I_K\otimes \phi^\top)^\top.
\end{align}
Next we deal with the second term, $\mathbb{E}\big[W J_\phi\Sigmabm_\epsilon J_\phi^\top W^\top\big]$. For each $(i,j)$, it follows that
\begin{align*}
\mathbb{E}[WJ_\phi\Sigmabm_\epsilon J_\phi^\top W^\top]_{ij}
&= \mathbb{E}[w_i^\top J_\phi\Sigmabm_\epsilon J_\phi^\top w_j] = \operatorname{Tr}\big(J_\phi \Sigmabm_\epsilon J_\phi^\top \mathbb{E}[w_j w_i^\top]\big) =\operatorname{Tr}\big(J_\phi\Sigmabm_\epsilon J_\phi^\top(\Sigmabm_{ji}+\mu_j\mu_i^\top)\big) \\
&= \operatorname{Tr}\big(J_\phi\Sigmabm_\epsilon J_\phi^\top \Sigmabm_{ij}\big)
+ \mu_i^\top J_\phi\Sigmabm_\epsilon J_\phi^\top \mu_j,
\end{align*}
which implies a compact matrix form
\begin{align}\label{eq:cov_U2}
\mathbb{E}[WJ_\phi\Sigmabm_\epsilon J_\phi^\top W^\top]
= W_{\mathrm{MAP}} J_\phi\Sigmabm_\epsilon J_\phi^\top W_{\mathrm{MAP}}^\top
+ \operatorname{Tr}_D\big[(I_K\otimes (J_\phi\Sigmabm_\epsilon J_\phi^\top))\Sigmabm_W\big].
\end{align}
Here $\operatorname{Tr}_D : \mathbb{R}^{KD\times KD} \to \mathbb{R}^{K\times K}$ denotes the operator that maps a block matrix to the matrix whose $(i,j)$-th entry is the trace of its $D\times D$ block $\Sigma_{ij}$.
Substituting \eqref{eq:cov_U1} and \eqref{eq:cov_U2} in \eqref{eq:cov_U} completes the proof.


\end{proof}

\subsection{Proof of 
~\Cref{prop:lb_kl}}\label{proof:proposition4.2}

The proof consists of two parts: (i) demonstrating that the LB surrogate $\pibm$ is a Dirichlet distribution
$\mathrm{Dir}_p(\pbm|\alphabm)$ with parameters $\alphabm$ matching the first two moments of $\pi$ up to $\mathcal{O}(\|\Sigma_z\|_F)$ and $\Ocal(\|\Sigma_z\|_F^{\frac{3}{2}})$ respectively; (ii) ; and (ii) establishing an upper bound on the KL divergence between $\pibm$ and $\tilde{\pibm}$.

Part (i). 
Let the first-order delta-method expansion of $\softmax$ around $\mubm_z$ be
$\pibm(\zbm)=\softmax(\mubm_z)+J_{\softmax}(\mubm_z)(\zbm-\mubm_z)+r(\zbm)$, where $J_{\softmax}(\cdot)$ is the Jacobian of $\softmax$ and the remainder $r(\zbm)$ is defined as $r(\zbm) = \int_0^1 (1-t) \, H_{\softmax}\big(\mubm_z + t(\zbm - \mubm_z)\big) \big[(\zbm - \mubm_z), (\zbm - \mubm_z)\big] \, dt$.
Here, $H_{\softmax}(\xbm)$ is the Hessian tensor of $\softmax$ evaluated at $x$, 
and the notation $H_{\softmax}(\xbm)[v,v]$ denotes the quadratic form $H_{\softmax}(\xbm)[v,v]_i := \sum_{j,k} \frac{\partial^2 \softmax_i(\xbm)}{\partial x_j \partial x_k} v_j v_k, \forall i=1,\dots,D$.

Under assumption (iii), the Hessian is uniformly bounded in a neighborhood of $\mubm_z$. Consequently, there exists a constant $M>0$ such that $\|H_{\softmax}(\xbm)\|_2\leq M,\forall\xbm\in B(\mubm_z)$ for some $M>0$. It follows that the remainder term satisfies $\|r(\zbm)\|_2\leq \frac{M}{2}\|\zbm-\mubm_z\|_2^2$ where the factor $\frac{1}{2}$ comes from the integral $\int_0^1(1-t)dt=\frac{1}{2}$. Since $\zbm\sim\Ncal(\mubm_z,\Sigmabm_z)$, we have that $\Ebb\|\zbm-\mubm_z\|_2^2=\operatorname{Tr}(\Sigmabm_z)\leq \sqrt{D}\|\Sigmabm_z\|_F$. Therefore, $\Ebb\|r(\zbm)\|_2\leq \frac{M}{2}\Ebb\|\zbm-\mubm_z\|_2^2\leq \frac{M\sqrt{D}}{2}\|\Sigmabm_z\|_F$, which implies that 
\begin{align}\label{eq:e_pi_upper}
    \|\Ebb[\pibm]-\softmax(\mubm_z)\|_2\leq \frac{M\sqrt{D}}{2}\|\Sigmabm_z\|_F.
\end{align}
Define $\delta_\pibm:=\pibm-\Ebb[\pibm]=J_{\softmax}(\mubm_z)(\zbm-\mubm_z)+\big(r(\zbm)-\Ebb[r(\zbm)] \big)$ and let $\tilde r(\zbm):=r(\zbm)-\Ebb[r(\zbm)], J:=J_{\softmax}(\mubm_z)$.
Then the covariance of $\pibm$ can be decomposed as
\begin{align}\label{eq:cov_pi}
    \Cov(\pibm) & = \Ebb[\delta_\pibm\delta_\pibm^\top] = \Ebb[ J(\zbm-\mubm_z)(\zbm-\mubm_z)^\top J^\top ] + \Ebb[ \tilde{r}(\zbm)\tilde{r}(\zbm)^\top]+2\Ebb[J(\zbm-\mubm_z)\tilde{r}(\zbm)^\top] \nonumber\\
    & = J\Sigmabm_z J^\top+\Ebb[ \tilde{r}(\zbm)\tilde{r}(\zbm)^\top]+2\Ebb[J(\zbm-\mubm_z)\tilde{r}(\zbm)^\top].
\end{align}
We now bound the second and the third terms in \eqref{eq:cov_pi}. First, by triangle inequality and the bound on $r(\zbm)$ becomes
\begin{align}\label{eq:r_tilde_upper}
    \|\Ebb[ \tilde{r}(\zbm)\tilde{r}(\zbm)^\top]\|_2 & \leq \Ebb\|\tilde{r}(\zbm)\|^2_2 \overset{\text{(I)}}{\leq} \Ebb \left[\left( \frac{M}{2}\|\zbm-\mubm_z\|_2^2+\frac{M\sqrt{D}}{2}\|\Sigmabm_z\|_F\right)^2\right] \nonumber\\
    & = \frac{M^2}{4} \Ebb\|\zbm-\mubm_z\|^4_2+\frac{M^2\sqrt{D}}{2}\|\Sigmabm_z\|_F\Ebb\|\zbm-\mubm_z\|_2^2+\frac{M^2D}{4}\|\Sigmabm_z\|_F^2 \nonumber\\
    & \overset{\text{(II)}}{\leq}  \frac{M^2}{4}\Big((\operatorname{Tr}(\Sigmabm_z))^2+2\|\Sigmabm_z\|_F^2\Big) +\frac{3M^2D}{4}\|\Sigmabm_z\|_F^2\leq \frac{M^2(D^2+3D+2)}{4}\|\Sigmabm_z\|_F^2
\end{align}
where (I) follows from the bound $\|\tilde{r}(\zbm)\|_2\leq \|r(\zbm)\|_2+\|\Ebb[r(\zbm)]\|_2\leq \frac{M}{2}\|\zbm-\mubm_z\|_2^2+\Ebb\|r(\zbm)\|_2\leq \frac{M}{2}\|\zbm-\mubm_z\|_2^2+\frac{M\sqrt{D}}{2}\|\Sigmabm_z\|_F$ and (II) holds because $\zbm - \mubm_z \sim \Ncal(0,\Sigmabm_z)$ implies $\Ebb\|\zbm-\mubm_z\|^4_2=(\operatorname{Tr}(\Sigmabm_z))^2+2\|\Sigmabm_z\|_F^2$. Indeed, letting $\xbm = \zbm - \mubm_z \sim \Ncal(0,\Sigmabm_z)$, we have $\Ebb\|\xbm\|^4_2=\sum_{i,j}\Ebb[x_i^2x_j^2]=\sum_{i,j}\left([\Sigmabm_z]_{ii}[\Sigmabm_z]_{jj}+2[\Sigmabm_z]_{ij}^2\right)$, where the last equality follows from Isserlis’ theorem. Therefore, $\Ebb\|\xbm\|^4_2=(\operatorname{Tr}(\Sigmabm_z))^2+2\|\Sigmabm_z\|_F^2$.

Next, for the third term of \eqref{eq:cov_pi}, using Cauchy–Schwarz inequality, we have that
\begin{align}\label{eq:J_r_tilde_upper}
    \|\Ebb[J(\zbm-\mubm_z)\tilde{r}(\zbm)^\top]\|_2 & \leq \|J\|_2\Ebb[\|\zbm-\mubm_z\|_2\|\tilde{r}(\zbm)\|_2]\leq \|J\|_2\sqrt{\Ebb\|\zbm-\mubm_z\|_2^2}\sqrt{\Ebb\|\tilde{r}(\zbm)\|_2^2} \nonumber\\
    & \leq \|J\|_2\sqrt[4]{D}\|\Sigmabm_z\|_F^{\frac{1}{2}}\cdot \sqrt{\frac{M^2(D^2+3D+2)}{4}\|\Sigmabm_z\|_F^2} \nonumber\\
    & = \frac{M(D^2+3D+2)^{\frac{1}{2}}D^{\frac{1}{4}}}{2}\|J\|_2\cdot\|\Sigmabm_z\|_F^{\frac{3}{2}},
\end{align}
where the last inequality follows from \eqref{eq:r_tilde_upper}. Combining \eqref{eq:cov_pi}--\eqref{eq:J_r_tilde_upper} yields 
\begin{align}\label{eq:cov_pi_upper}
    \|\Cov(\pibm)-J\Sigmabm_z J^\top\|_2 \leq M(D^2+3D+2)^{\frac{1}{2}} D^{\frac{1}{4}}\|J\|_2\cdot\|\Sigmabm_z\|_F^{\frac{3}{2}}+ \frac{M^2(D^2+3D+2)}{4}\|\Sigmabm_z\|_F^2
\end{align}

We choose $\alphabm$ of the Dirichlet distribution $\mathrm{Dir}_p(\pbm|\alphabm)$ such that its mean and covariance are equal to 
\begin{align}\label{eq:alpha_set}
    [\softmax(\mu_z)]_i=\frac{\alpha_i}{\alpha_0}, \quad [J\Sigmabm_zJ^\top]_{ij}=\frac{\alpha_i(\alpha_0-\alpha_i)}{\alpha_0^2(\alpha_0+1)}\,\mathbf{1}_{i=j}
-\frac{\alpha_i\alpha_j}{\alpha_0^2(\alpha_0+1)}\,\mathbf{1}_{i\neq j},
\end{align}
where $\alpha_0=\sum_i\alpha_i$. From \eqref{eq:e_pi_upper} and \eqref{eq:cov_pi_upper}, we have that 
\begin{align*}
\|\Ebb[\pibm]-\softmax(\mubm_z)\|_2=\Ocal(\|\Sigmabm_z\|_F), \quad \|\Cov(\pibm)-J\Sigmabm_z J^\top\|_2=\Ocal(\|\Sigmabm_z\|_F^{\frac{3}{2}}).
\end{align*}

Part (ii). 
Under assumptions (ii) and (iii), the distributions $P_\pibm$ and $P_{\tilde{\pibm}}$ admit 
twice-differentiable densities with bounded curvature in a neighborhood of 
$\Ebb[\pi]$ inside the simplex $\Delta^{K-1}$. Let $\mubm_\pi := \Ebb[\pibm]=\softmax(\mubm_z)$. 
By a second-order Taylor expansion of the log-density around $\mubm_\pi$, 
for any $\xbm$ near $\mubm_\pi$ we have
\begin{align}\label{eq:log-p-exp}
    \log p_\pibm(\xbm) - \log p_{\tilde{\pibm}}(\xbm) 
= g^\top (\xbm - \mubm_\pi) + \frac{1}{2} (\xbm - \mubm_\pi)^\top H (\xbm - \mubm_\pi) + R(\xbm),
\end{align}
where $g:=\nabla_\xbm(\log p_\pibm-\log_{p_{\tilde{\pibm}}})|_{\xbm=\mubm_z} \in \mathbb{R}^K$ and $H:=\nabla^2_\xbm(\log p_\pibm-\log_{p_{\tilde{\pibm}}})|_{\xbm=\mubm_z} \in \mathbb{R}^{K \times K}$ are bounded, 
and the remainder satisfies $|R(\xbm)| = \mathcal{O}(\|\xbm - \mubm_\pi\|_2^3)$. 
This is precisely the \emph{LAN-style expansion}, which formalizes the idea that the log-likelihood ratio is quadratic in small deviations from the mean, with controlled higher-order terms \cite{cam1960locally, van2000asymptotic}. For $X\sim P_\pibm$, the KL divergence can then be written as
\begin{align*}
    \mathrm{KL}(P_\pibm \| P_{\tilde{\pibm}}) = \Ebb[\log p_\pibm(X) - \log p_{\tilde{\pibm}}(X)].
\end{align*}
Substituting \eqref{eq:log-p-exp} into $\mathrm{KL}(P_\pibm \| P_{\tilde{\pibm}})$ gives
\begin{align*}
    \mathrm{KL}(P_\pibm \| P_{\tilde{\pibm}}) & = \Ebb[g^\top (X - \mubm_\pi)] + \frac{1}{2} \Ebb[(X - \mubm_\pi)^\top H (X - \mu_\pi)] + \Ebb[R(X)],
\end{align*}
where the first term vanishes since $\Ebb[X - \mu_\pi] = 0$. Let $\Delta_{\mathrm{mom}} = \|\mathbb{E}[\pibm]-\mathbb{E}[\tilde{\pibm}]\|_2 + \|\mathrm{Cov}(\pibm)-\mathrm{Cov}(\tilde{\pibm})\|_F$ be the moment discrepancy that measures the difference between the first two moments (mean and covariance) of $\pibm$ and $\tilde\pibm$. The upper bound of the KL divergence becomes
\begin{align*}
    \mathrm{KL}(P_\pibm \| P_{\tilde{\pibm}}) & \overset{\text{(III)}}{\leq} \frac{1}{2}\Big|\Ebb[(X - \mubm_\pi)^\top H (X - \mubm_\pi)]\Big|+|\Ebb[R(X)]| \leq \frac{1}{2}\|H\|_2\Ebb\|X-\mubm_\pi\|^2_2+C_1\Ebb[\|X-\mubm_z\|_2^3] \\
    & \overset{\text{(IV)}}{\leq}  \frac{1}{2}\|H\|_2\Delta_{\mathrm{mom}}+C_1\left( \Ebb[\|X-\mubm_z\|_2^2]\right)^{\frac{3}{2}} \leq \frac{1}{2}\|H\|_2\Delta_{\mathrm{mom}}+C_1(\Delta_{\mathrm{mom}})^{\frac{3}{2}},
\end{align*}
where (III) is true for some $C_1>0$ since $|R(\xbm)| = \mathcal{O}(\|\xbm - \mubm_\pi\|_2^3)$ and (IV) holds by using Jensen's inequality. With the choice of $\alphabm$ given in \eqref{eq:alpha_set}, we have $\tilde{\pibm}=\softmax(\mubm_z)$ and $\Cov(\tilde{\pibm})=J\Sigmabm_z J^\top$, which leads to $\Delta_{\mathrm{mom}}=\frac{M\sqrt{D}}{2}\|\Sigmabm_z\|_F+M(D^2+3D+2)^{\frac{1}{2}} D^{\frac{1}{4}}\|J\|_2\cdot\|\Sigmabm_z\|_F^{\frac{3}{2}}+ \frac{M^2(D^2+3D+2)}{4}\|\Sigmabm_z\|_F^2$ from \eqref{eq:e_pi_upper} and \eqref{eq:cov_pi_upper}. Plugging it into the inequality above, we conclude that 
\begin{align*}
    \mathrm{KL}(P_\pibm\|P_{\tilde{\pibm}})\leq \frac{M\sqrt{D}\|H\|_2}{2}\|\Sigmabm_z\|_F+\mathcal{O}(\|\Sigmabm_z\|_F^{\frac{3}{2}})=C(\mubm_z)\|\Sigmabm_z\|_F+\mathcal{O}(\|\Sigmabm_z\|_F^{\frac{3}{2}}),
\end{align*}
where $C(\mubm_z):=\frac{M\sqrt{D}\|H\|_2}{2}=\frac{1}{2}M\sqrt{D}\|\nabla^2_\xbm(\log p_\pibm-\log_{p_{\tilde{\pibm}}})|_{\xbm=\mubm_z}\|_2>0$ is a constant independent of $\Sigmabm_z$.

\subsection{LBS Certification under $\ell_\infty$ Perturbations}
\label{l_infinity}
\begin{figure}[htbp]
    \centering
    \includegraphics[width=0.8\textwidth]{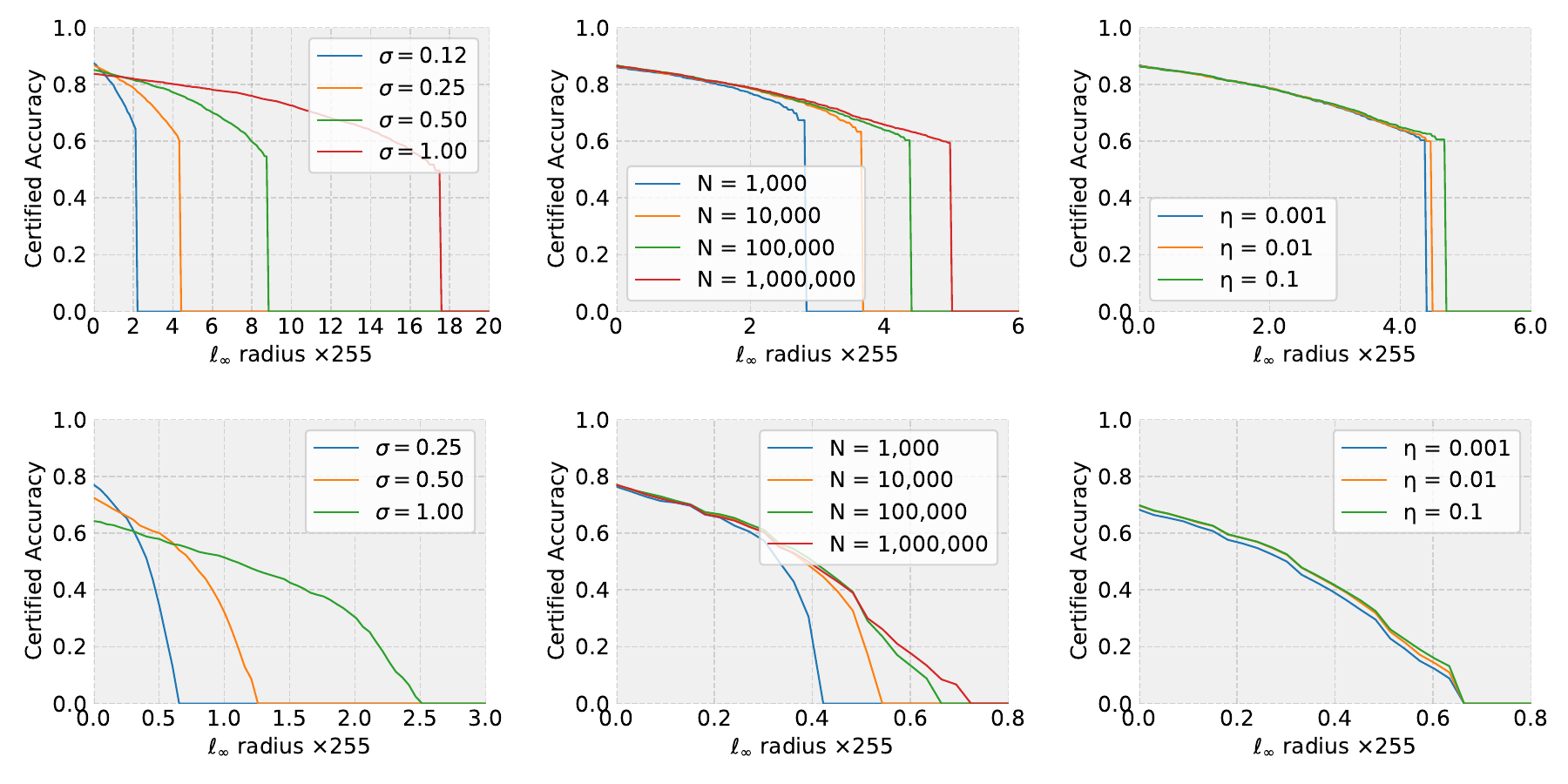}
    \caption{Certified accuracy of LBS on CIFAR-10 (top row) and ImageNet (bottom row). Left: with the different perturbation levels $\sigma$ and $N$ = 10,0000, $\eta$ =0.001. Middle: with the different number of samples $N$ and $\sigma$ = 0.25, $\eta$ =0.001. Right: with the different confidence levels $\eta$ and $N$ =100,000, $\sigma$ = 0.25. }
    \label{fig:LBS_l_infinity}
\end{figure}

We evaluate the robustness of LBS under an $\ell_\infty$ adversarial threat model.
Given a lower bound $\underline{p_A}$ on the smoothed class probability, the certified $\ell_\infty$ radius is obtained by mapping the probability bound to a robustness radius via the level-set differential formulation introduced by~\citet{yang2020randomized}.
As shown in Figure~\ref{fig:LBS_l_infinity}, we examine the trade-off between certified accuracy and radius for LBS under different smoothing noise levels $\sigma$, MC sample sizes $N$, and confidence levels $\eta$.
For visualization, the certified $\ell_\infty$ radius is reported in pixel scale by multiplying the normalized radius by $255$, following the convention of~\citet{yang2020randomized}.

Across all settings, LBS exhibits consistent certification behavior under $\ell_\infty$ perturbations.
Despite the inherently small certified radii imposed by theoretical limitations under $\ell_\infty$ threat models, LBS provides stable and well-calibrated robustness guarantees within this constrained regime.


\subsection{Effect of Different Model Architectures}
\label{model_architectures}

RS is inherently model-agnostic and can, in principle, be applied to arbitrary base classifiers~\citep{cohen2019certified,jia2019certified, yang2020randomized, li2023sok, scholten2023hierarchical, wang2024drf}. 
To examine whether this property carries over to LBS in practice, we evaluate LBS under different network architectures, 
including DenseNet-161 and EfficientNet-B0, on CIFAR-10.

\begin{wrapfigure}[8]{r}{0.48\textwidth}
    \vspace{-4mm}   
    \centering
    \resizebox{0.48\textwidth}{!}{
    \begin{tabular}{lcccccc}
        \toprule
        $\ell_1$ radius $\rightarrow$ & 0.5 & 1.0 & 1.5 & 2.0 & 2.5 & 3.0 \\
        \midrule
        DenseNet-161      & 88.6 & 86.5 & 84.9 & 83.7 & 82.7 & 81.8 \\
        \midrule
        EfficientNet-B0  & 87.9 & 85.7 & 83.7 & 82.3 & 81.2 & 80.2 \\
        \bottomrule
    \end{tabular}}
    \captionof{table}{Certified accuracy attained by LBS across architectures under different $\ell_1$ radii (Uniform noise level $\lambda$), $N=100{,}000$, and $\eta=0.001$ on CIFAR-10.}
    \label{tab:diff_archi}
    \vspace{-4mm}
\end{wrapfigure}

Figure~\ref{fig:LBS_diff_archi} and Table~\ref{tab:diff_archi} show the certified accuracy of LBS under both $\ell_2$ and $\ell_\infty$ threat models across a range of smoothing parameters $\sigma$ and $\ell_1$ radii.
Across both architectures, LBS exhibits consistent certification behavior, with similar accuracy--radius trade-offs and stable performance trends as the perturbation magnitude increases.
These results indicate that LBS does not rely on architecture-specific properties and preserves the model-agnostic nature of RS, while delivering reliable and comparable certified robustness across diverse network designs.

\begin{figure}[htbp]
    \centering
    \includegraphics[width=0.98\textwidth]{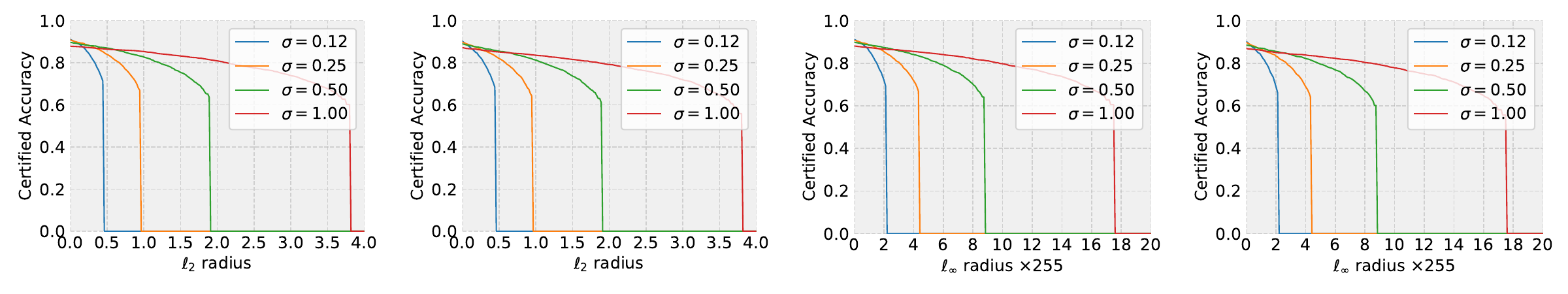}
    \caption{Certified accuracy of LBS across architectures on CIFAR-10 with varying $\sigma$, $N=100{,}000$, $\eta=0.001$.
Left two panels: $\ell_2$ certification; right two panels: $\ell_\infty$ certification (radius $\times 255$).
DenseNet-161 (left column) and EfficientNet-B0 (right column).}
    \label{fig:LBS_diff_archi}
\end{figure}

\subsection{LBS Certification with Noise Augmentation}\label{LBS_with_noise_augmentation}

We further investigate the behavior of LBS when the base classifier is trained with noise augmentation. Specifically, we employ a ResNet-110 classifier trained following the procedure of \citet{cohen2019certified}, using the same noise level for data augmentation during training. Figure~\ref{fig:LBS_noisetrain} reports the certified accuracy–radius trade-offs of LBS under different smoothing noise levels $\sigma$, MC sample sizes $N$, and confidence levels $\eta$. 

The results show that, the certified accuracy attained by LBS exhibits a modest degradation compared to the setting without noise augmentation (see ~\Cref{performance_LBS}). This degradation can be attributed to the interaction between noise-augmented training and the LB approximation.
Noise augmentation during training reshapes the logit distribution, typically reducing class separability and increasing posterior variance.
As a result, the induced logit distribution deviates from the locally concentrated Gaussian regime in which the Laplace-bridge approximation is most accurate, leading to more conservative Dirichlet surrogates and lower certified accuracy.

However,  RS is theoretically applicable to any base classifier and is often framed as a post-hoc defense that does not require modifications to the training procedure.
We found that in practice, achieving competitive certified accuracy with RS and its variants typically relies on noise-augmented training, which couples the certification mechanism to the training process.
This reliance weakens one of the key conceptual advantages of RS as a post-hoc method. Therefore, LBS is explicitly designed to address the limitation of RS and operate without noise-augmented training. 


\begin{figure}[htbp]
    \centering
    \includegraphics[width=0.8\textwidth]{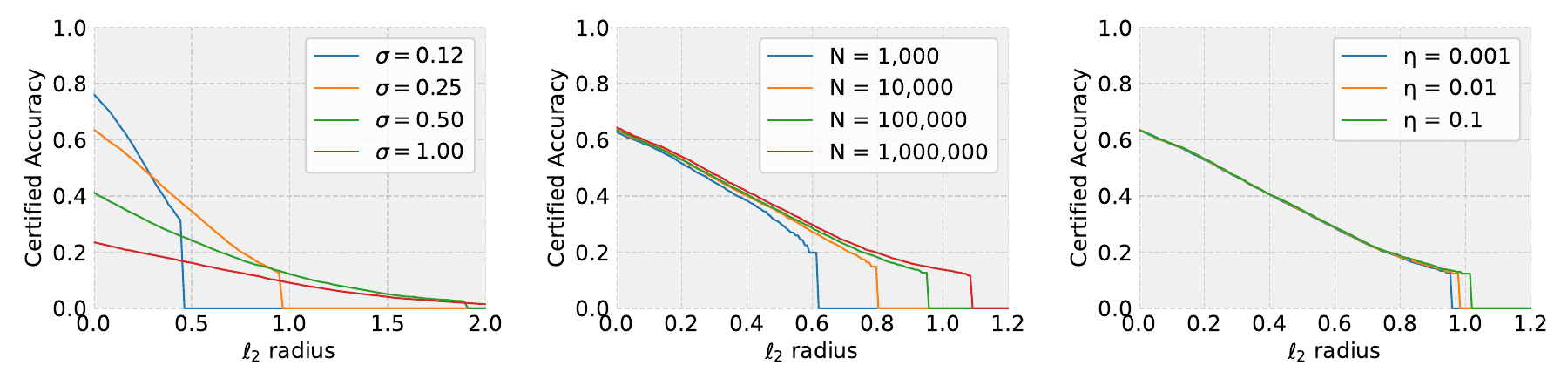}
    \caption{Certified accuracy of LBS on CIFAR-10. Left: varying $\sigma$, with $N=100{,}000$ and $\eta=0.001$. Middle: varying $N$, with $\sigma=0.25$ and $\eta=0.001$. Right: varying $\eta$, with $N=100{,}000$ and $\sigma=0.25$. }
    \label{fig:LBS_noisetrain}
\end{figure}

\subsection{On-device LBS Certification}\label{edge_app}
In Section~\ref{edge}, we conduct experiments on two representative on-device platforms with different computational profiles: an NVIDIA Jetson Orin Nano and a Raspberry~Pi~4. The Jetson Orin Nano runs L4T 36.4.3 / JetPack 6 with Linux 5.15.148-tegra, a 6-core ARM Cortex-A78AE CPU, aarch64 architecture, approximately 7.4 GiB RAM, and PyTorch 2.8.0. The Raspberry~Pi~4 is a Model B Rev 1.5 running Debian 6.12.25+rpt-rpi-v8 with a 4-core ARM Cortex-A72 CPU, aarch64 architecture, approximately 3.7 GiB RAM, and CPU-only PyTorch 2.7.0.

On both platforms, we deploy the ResNet110, SqueezeNet, and ShuffleNet classifiers trained on CIFAR-10 following the protocol of \citet{cohen2019certified}, without noise augmentation. The checkpoints are loaded directly on the device, and the full certification pipeline is executed locally, including model inference, LBS certification, and per-sample certified-radius estimation.

Figure~\ref{fig:LBS_jetson} reports the certified accuracy of LBS (SqueezeNet backbone) on two on-device platforms, with the left columns corresponding to NVIDIA Jetson Orin Nano and the right columns to Raspberry~Pi~4.
Across both devices, increasing the number of samples $N$ improves certified accuracy, with diminishing returns beyond $N{=}10{,}000$, while smaller values of $\eta$ produce more conservative but stable certificates.
These trends are consistent across platforms, indicating stable on-device certification behavior.

Table~\ref{tab:l1_LBS_edge} further reports the certified accuracy of LBS under uniform noise for $\ell_1$ certification on CIFAR-10.
As the noise level $\lambda$ increases, certified accuracy decreases gradually on both Jetson Orin Nano and Raspberry~Pi~4, reflecting the expected robustness–accuracy trade-off.
Importantly, the performance gap between the two devices remains small across all radii, demonstrating that LBS maintains consistent $\ell_1$ certified robustness under uniform noise in on-device settings.

Figure~\ref{fig:LBS_l_infinity_ondevice} evaluates LBS under the $\ell_\infty$ threat model.
Increasing the noise scale $\sigma$ enlarges the certified $\ell_\infty$ radius at the cost of reduced certified accuracy, whereas increasing $N$ yields tighter and more stable certification, again with diminishing gains beyond $N{=}10{,}000$.

Despite differences in compute and memory capacity, Jetson Orin Nano and Raspberry~Pi~4 exhibit similar certification behavior across under $\ell_p$ settings.
The observed trends are consistent with those obtained on server-class hardware and support the applicability of LBS in resource-constrained, on-device environments.

\begin{figure}[htbp]
    \centering
    \includegraphics[width=0.98\textwidth]{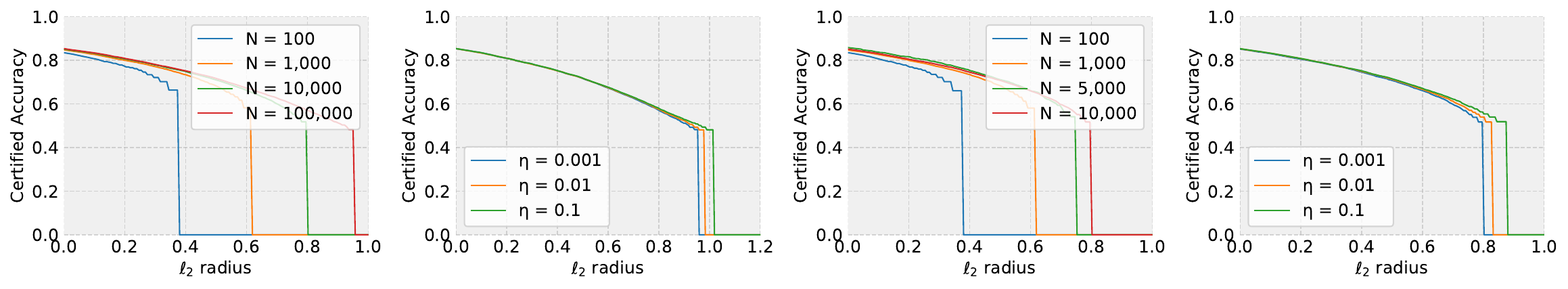}
    \caption{Certified accuracy of LBS on CIFAR-10 on two edge platforms under $\ell_2$ perturbation. Left two: NVIDIA Jetson Orin Nano; right two: Raspberry~Pi~4. For each platform, the left panel varies $N$; the right panel varies $\eta$.}
    \label{fig:LBS_jetson}
\end{figure}

\begin{figure}[htbp]
    \centering
    \includegraphics[width=0.98\textwidth]{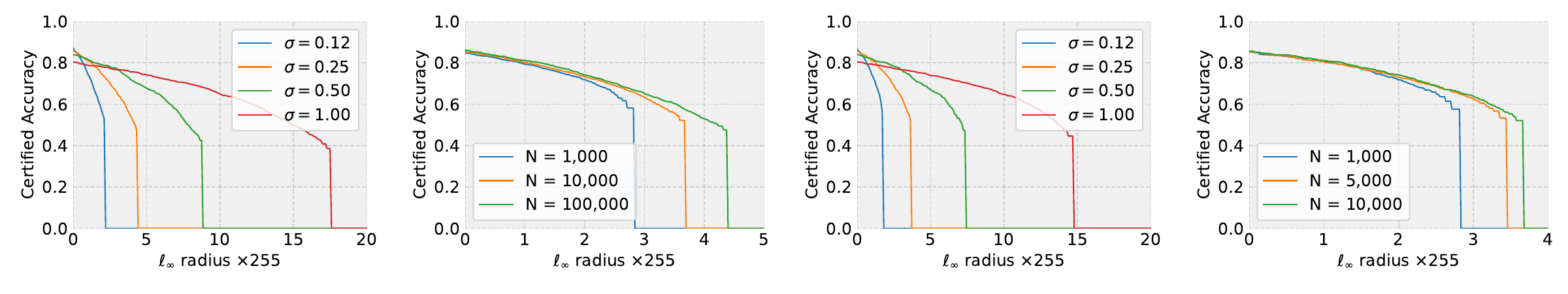}
    \caption{Certified accuracy of LBS on CIFAR-10 on two edge platforms under $\ell_\infty$ perturbation, with radius scaled by $\times 255$. Left two columns: NVIDIA Jetson Orin Nano; right two columns: Raspberry~Pi~4. For each platform, the left panel varies $\sigma$ with fixed $N$ = $100{,}000$ on Jetson and $10{,}000$ on Pi; the right panel varies $N$ with fixed $\sigma$ = $0.25$.}
    \label{fig:LBS_l_infinity_ondevice}
\end{figure}

\begin{table}[h]
\centering
\resizebox{0.48\textwidth}{!}{
\begin{tabular}{lcccccc}
\toprule
$\ell_1$ radius $\rightarrow$ & 0.5 & 1.0 & 1.5 & 2.0 & 2.5 & 3.0 \\
\midrule
Jetson Orin Nano & 83.8 & 80.1 & 79.1& 77.0 & 76.2 & 75.7 \\
\midrule
Raspberry~Pi~4 & 85.0 & 82.7 & 81.0 & 79.3 & 78.3 & 77.0 \\
\bottomrule
\end{tabular}}
\caption{Certified accuracy of LBS under different uniform noise levels $\lambda$ on CIFAR-10. }
\label{tab:l1_LBS_edge}
\end{table}

\subsection{Use of Large Language Models (LLMs)}\label{llms}
We used LLMs for editorial assistance (clarity, grammar, wording, and minor reorganizations). LLMs were not used to generate ideas, design experiments, or manipulate results, or draft related-work claims. The authors take full responsibility for all content.

\end{document}